\documentclass{article}

    \PassOptionsToPackage{numbers, compress}{natbib}

 \usepackage[main, final]{neurips_2026}

\usepackage[utf8]{inputenc} 
\usepackage[T1]{fontenc}    
\usepackage{hyperref}       
\usepackage{url}            
\usepackage{booktabs}       
\usepackage{amsfonts}       
\usepackage{nicefrac}       
\usepackage{microtype}      
\usepackage{xcolor}         
\usepackage{amsmath}
\usepackage{graphicx}  
\usepackage{float}  
\usepackage{subfigure}  
\usepackage{svg}
\usepackage{algorithm}
\usepackage{algorithmic}
\usepackage{multirow}
\usepackage{enumerate}
\usepackage{amsthm}
\usepackage{multicol}
\usepackage{mathtools}
\usepackage{enumitem}

\newtheorem{theorem}{Theorem}

\newtheorem{proposition}{Proposition}

\title{Variational Consequence-Driven Offline Reinforcement Learning}

\author{%
  Wen Jiang$^{1,2}$, 
  Ke Jiang$^{1,2,3}$, 
  Jin Wang$^{1,2}$, 
  Xiaoyang Tan$^{1,2}$ \\
  $^1$College of Computer Science and Technology, Nanjing University of Aeronautics and Astronautics \\
  $^2$MIIT Key Laboratory of Pattern Analysis and Machine Intelligence \\
  $^3$Graduate School of Information Science and Technology, The University of Osaka \\
  \texttt{\{darren.jum, ke\_jiang, jinw1999, x.tan\}@nuaa.edu.cn}
}

\begin{document}

\maketitle

\begin{abstract}

Offline reinforcement learning (RL) typically mitigates \textit{distribution shift} by imposing divergence constraints on the action distributions of the learned and behavior policies. However, this paradigm cannot ensure the underlying dynamic consistency structure: even minor action deviations can lead to drastically different transitions, inadvertently driving the agent away from the offline dataset.
To address this, we explore policy constraints through their induced state-consequence distributions rather than relying on pure action-distribution matching. Nevertheless, directly enforcing such consequence-aware constraints is computationally intractable without access to the true underlying dynamics model.
To bypass this issue, we propose Structural Behavior Regularization (StBR), which replaces the intractable objective with a closed-form latent divergence. Structuring this latent space to reflect transition dynamics renders the latent divergence a theoretically bounded surrogate for the true consequence-distribution divergence.
By regularizing the learned policy with this latent divergence, StBR adaptively tightens the behavior constraint in dynamics-sensitive regions, effectively preventing deviations outside the offline dataset. Empirically, we demonstrate that StBR achieves superior performance across diverse offline RL benchmarks.



\end{abstract}

\section{Introduction}

Offline Reinforcement Learning (RL) aims to learn effective policies entirely from a fixed dataset collected by unknown behavior policies \citep{batchrl}. This paradigm holds immense promise for real-world applications where data collection is costly or dangerous, such as robotics~\citep{sinha2022s4rl}, healthcare~\citep{tang2022leveraging}, and autonomous driving~\citep{fang2022offline}. However, standard RL algorithms often fail in this setting due to the challenge of \textit{distribution shift}, where the learned policy targets out-of-distribution (OOD) actions that are erroneously overestimated by the value function~\citep{orlreview}.

To avoid OOD actions, a dominant approach is \textit{behavior regularization}, which restricts the learned policy to remain close to the behavior policy via divergence measures. Early methods employ simple Mean Squared Error (MSE)~\cite{td3bc} for action cloning, which inherently assume unimodal distributions and often induce over-conservatism~\citep{bear} or mode collapse~\citep{cped} on complex datasets. 
To alleviate such drawbacks, subsequent approaches incorporate generative models (e.g., VAE~\citep{spot}, flow-GAN~\citep{cped}) and distribution divergence (e.g., Kullback-Leibler (KL)~\citep{osr}, Maximum Mean Discrepancy (MMD)~\citep{bear}) to constrain the learned policy within the behavior support, successfully generalizing beyond mere data points. While this paradigm shift brings initial improvements, the limited expressivity of these generative models eventually bottlenecks performance, prompting recent methods to leverage \textit{diffusion}~\citep{eql} and \textit{flow matching}~\citep{fql} for more precise modeling.
However, these methods fundamentally remain rooted in direct action distribution matching, which is inherently agnostic to the underlying dynamic consistency. In practice, due to function approximation and finite dataset coverage, estimation errors in the learned action distribution are practically inevitable. As illustrated in Figure~\ref{fig:sbr_concept} (left), even minor action deviations can lead to drastically different state transitions. Consequently, the agent inadvertently deviates from the offline dataset, particularly in dynamics-sensitive regions.

To address this, we explore policy constraints through their induced state-consequence distributions rather than relying on pure action-distribution matching. However, directly enforcing such consequence-aware constraints is computationally intractable: it requires access to the true, unknown environment dynamics, and explicitly estimating dynamics density from offline data is notoriously infeasible (as shown in Figure~\ref{fig:sbr_concept}, right). To bypass this issue, we replace the direct evaluation of consequence constraints with a tractable surrogate divergence in latent space. 
Specifically, we propose Structural Behavior Regularization (StBR)\footnote{Code is available at: \url{https://github.com/DarrenJum/StBR}}, a novel framework that approximates the intractable objective via a closed-form KL divergence computed within a structured, consequence-aware Gaussian latent space (Figure~\ref{fig:sbr_concept}, middle).

To effectively construct this latent space, StBR introduces a State Belief Gaussian Fitting technique grounded in the Variational Information Bottleneck principle~\citep{vib}. By mapping state-action pairs into the latent space, State Belief Gaussian Fitting extracts the essential consequence-aware information and structures the aggregated behavior embeddings into a unimodal Gaussian distribution, serving as a state belief. It is crucial to note that this unimodality is not an assumption about the real-world action or consequence distributions; rather, it is purely a geometric constraint imposed on the latent space. We theoretically establish that the discrepancy between the true consequence divergence and this surrogate is provably bounded, guaranteeing a reliable approximation. 
Furthermore, because this latent space captures environmental dynamics information, minor action deviations naturally trigger drastic shifts in the latent representations.
By regularizing the learned policy with this latent divergence, StBR translates these representation shifts into an adaptive penalty, effectively tightening the behavior constraint in dynamics-sensitive regions. Empirically, StBR achieves superior performance across diverse offline RL benchmarks.

Our main contributions are summarized as follows:
\begin{itemize} [itemsep=3pt, topsep=0pt, parsep=0pt, partopsep=0pt]
    \item We propose Structural Behavior Regularization (StBR), a novel framework that enforces conservatism through the policy's induced consequences via a surrogate closed-form Gaussian latent divergence without dynamics density estimation.
    \item We introduce State Belief Gaussian Fitting to impose structure on the latent representations, and theoretically demonstrate that the proposed surrogate objective serves as a bounded approximation of the intractable consequence divergence.
    \item Empirically, StBR achieves superior performance across diverse offline RL benchmarks.
\end{itemize}

\begin{figure}
  \centering
  \includegraphics[width=1.0\columnwidth]{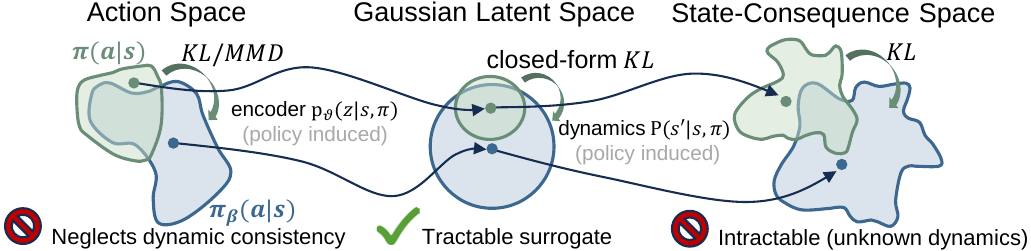} 
  \caption{Conceptual illustration of divergence-based behavior regularization. Purely action distribution matching (left) neglects dynamic consistency, whereas ideal consequence-aware constraints (right) are computationally intractable without access to the true dynamics. Our StBR (middle) employs a Gaussian latent space as a tractable, closed-form surrogate of consequence-aware constraints.}
  \label{fig:sbr_concept} 
\end{figure}

\section{Preliminaries and Motivation}

\subsection{Reinforcement Learning}
In RL, we generally model the environment as a Markov Decision Process (MDP) $\mathcal{M}=(\mathcal{S},\mathcal{A},P,R,\gamma,\rho_0)$, with state space $\mathcal{S}$, action space $\mathcal{A}$, transition dynamics $P:\mathcal{S}\times\mathcal{A}\to\Delta(\mathcal{S})$, reward function $R:\mathcal{S}\times\mathcal{A}\to[ 0, R_{\max }]$, discount factor $\gamma \in [0, 1)$, and initial state distribution $\rho_0$, where $\Delta(\cdot)$ denotes the probability simplex. The agent interacts with the environment and seeks a policy $\pi:\mathcal{S}\to\Delta(\mathcal{A})$ to maximize the expected discounted return $J(\pi):=\mathbb{E}_{s_0\sim\rho_0,a_t\sim\pi(\cdot|s_t),s_{t+1}\sim P(\cdot|s_t,a_t)}[\sum_{t=0}^\infty\gamma^t R_t(s_t,a_t)]$. The objective of reinforcement learning is to find an optimal policy: $\pi^\star:=\operatorname{argmax}_\pi J(\pi).$ It will be useful to define the \textit{policy-specific} state-action value function (Q-value function $Q^{\pi}:\mathcal{S}\times\mathcal{A}\to\mathbb{R}$) as $Q^{\pi}(s,a)=\mathbb{E}_\pi[\sum_{t=0}^{\infty}\gamma^{t}R_t(s_t,a_t)|s_{0}=s,a_{0}=a]$. 
Then we define the discounted state occupancy for convenience $d^\pi(s) = (1-\gamma)\sum^{\infty}_{t=0}\gamma^t Pr(s_t=s; \pi,\rho_0)$. 

In large and complex state-action spaces, function approximation is typically employed to model the value network $Q_\theta$ with parameters $\theta$ and the policy network $\pi_\phi$ with parameters $\phi$. In practice, an off-policy iterative approach commonly initializes parameters arbitrarily, and iteratively updates the parameterized function with an experience replay dataset $\mathcal{D}$ and a target network $Q_{\bar\theta}$~\citep{dqn}:
	\begin{align}
        \mathcal{L}_Q(\theta) =\underset{(s,a,r,s') \sim \mathcal{D}}{\mathbb{E}} \left[ Q_\theta(s, a) - r - \gamma Q_{\bar\theta}(s', \pi_\phi) \right]^2,
      \label{eq:empirical_T}
	\end{align}
where $Q_{\bar\theta}(s',\pi_\phi)$ is shorthand for $\mathbb{E}_{a'\sim \pi_\phi(\cdot|s')}[Q_{\bar\theta}(s' , a')]$. The policy $\pi_\phi$ can be updated by:
	\begin{align}
        \mathcal{L}_\pi(\phi)=\mathbb{E}_{s\sim\mathcal{D}, a\sim \pi_\phi(\cdot|s)}[-Q_\theta(s,a)].
      \label{eq:empirical_pi}
	\end{align}

\subsection{Offline Reinforcement Learning}
In offline RL, the agent has no direct interactive access to the environment but only to a fixed dataset $\mathcal{D} = \{(s_i, a_i, r_i, s'_i)\}_{i=1}^N$ collected by an unknown behavior policy $\pi_\beta$. The fundamental challenge in offline RL is the \textit{distribution shift}: querying the learned value function on OOD actions can lead to severe overestimation.

\subsection{Motivation: Rethinking Behavior Conservatism}
A prominent category of offline RL methods follows the key insight of maintaining conservatism in policy learning, which constrains the learned policy to remain close to the dataset by avoiding OOD actions. Existing methods typically achieve this by regularizing the action distribution of the learned policy $\pi_\phi$ toward the behavior policy $\pi_\beta$ via a divergence measure:
\begin{align}
     \mathcal{L}_\pi(\phi)=\mathbb{E}_{s\sim\mathcal{D}, a\sim \pi_\phi(\cdot|s)}\left[-Q_\theta(s,a) + \lambda D\big(\pi_\phi(a|s), \pi_\beta(a|s)\big)\right],
\end{align}
where $\lambda$ is the Lagrange multiplier, $D$ is a chosen divergence measure, such as Kullback-Leibler (KL) divergence in BRAC \& AWAC~\citep{brac,awac}, Maximum Mean Discrepancy (MMD) in BEAR~\citep{bear}, or Mean Squared Error (MSE) in TD3BC~\citep{td3bc}. However, enforcing regularizations purely on the action distribution is blind to the underlying dynamic consistency structure, leading to a critical issue:

\textbf{Consequence-shift vulnerability:} In dynamics-sensitive regions, even minor action deviations that incur a minimal divergence penalty can trigger drastically different transitions, inadvertently propelling the agent into OOD next states (consequences).

This limitation suggests that distribution shift is a problem related to \textit{dynamic consistency}. Therefore, rather than constraining action-distributions, we propose constraining their induced state-consequence distributions $P_\pi(s'|s) := \mathbb{E}_{a \sim \pi(\cdot|s)}P(s'|s, a)$. We formulate this constraint using KL divergence, as its zero-forcing property strictly penalizes visiting out-of-support states, naturally preventing OOD extrapolation. Our ideal objective is thus:
\begin{align}\label{eq:cc}
    \mathcal{L}_\pi(\phi)= \mathbb{E}_{s\sim\mathcal{D}, a\sim \pi_\phi(\cdot|s)}\left[-Q_\theta(s,a) + \lambda D_{\mathrm{KL}}\big(P_{\pi_\phi}(s'|s) \| P_{\pi_\beta}(s'|s)\big)\right].
\end{align}  

\textbf{The Tractability Challenge.} Directly optimizing Eq.~\eqref{eq:cc} is computationally intractable in the standard offline RL setting, as enforcing such a constraint requires explicit density estimation of the true underlying dynamics $P$. This intractability directly motivates our Structural Behavior Regularization (StBR) framework, detailed in the following section.


\section{Method}
In this section, we present the Structural Behavior Regularization (StBR) framework. We first introduce a latent variable to derive a tractable surrogate divergence (Section~\ref{sec:3.1}). Next, we propose the State Belief Gaussian Fitting (SBGF) technique to structure this latent space (Section~\ref{sec:3.2}), provide theoretical guarantees for this approximation (Section~\ref{sec:theory_bound}), and finally integrate the surrogate divergence into offline RL policy optimization (Section~\ref{sec:3.4}).

\subsection{Tractable Surrogate via Gaussian Latent Divergence}\label{sec:3.1}
We begin by introducing a latent variable $z$ as a sufficient statistic of the state-action pair $(s,a)$ for predicting the next state $s'$. Specifically, we assume the dynamics follow the causal diagram:
\begin{align}
    (s,a) \to z \to s'.
\end{align}
This causal structure implies conditional independence, $(s,a) \perp \! \! \! \perp s' | z$ which naturally decouples the transition dynamics as $P(s'|s,a) = \int_z p(s'|z)p(z|s,a) \, dz$. Crucially, this decomposition provides a theoretical motivation formalized in Proposition~\ref{proposition1}. The detailed proof can be found in Appendix~\ref{sec:app_p1}. Let the subscript $\pi$ denote expectation over $a \sim \pi(\cdot|s)$ (e.g., $p_\pi(z|s) = \mathbb{E}_{a \sim \pi}[p(z|s,a)]$).
\begin{proposition}[Surrogate Consequence Divergence]\label{proposition1}
Assume $z$ is a latent variable acting as a sufficient statistic for the state-action pair $(s, a)$ to predict the next state $s'$, forming the causal diagram $(s, a) \to z \to s'$. For any policy $\pi$ and behavior policy $\pi_\beta$, the KL divergence between their consequence state distributions is upper-bounded by the KL divergence between their corresponding latent distributions:
\begin{align}\label{eq:surrogate}
    D_{\mathrm{KL}} \big( P_{\pi}(s'|s) \| P_{\pi_\beta}(s'|s) \big) \le D_{\mathrm{KL}} \big( p_{\pi}(z|s) \| p_{\pi_\beta}(z|s) \big).
\end{align}
\end{proposition}



Proposition~\ref{proposition1} provides a theoretical foundation for enforcing consequence conservatism entirely within the latent space. To make this distribution matching computationally tractable, we structurally design the representation bottleneck as a Gaussian latent space. This choice effectively transforms the intractable problem into an analytical, closed-form KL divergence computation between two Gaussian distributions in the latent space. Consequently, the burden of intractable dynamics density estimation is shifted to a tractable representation learning task, which we address in the next section.

\subsection{State Belief Gaussian Fitting}\label{sec:3.2}

To implement the representation learning task outlined above and enforce the causal structure $(s,a) \to z \to s'$, we build upon the Variational Information Bottleneck (VIB) principle~\citep{vib}. Specifically, our objective is to maximize the mutual information between the latent representation $Z$ and the next state $S'$, while minimizing the information between $Z$ and the input pair $(S, A)$:
\begin{align}
    \max I(Z; S') - \beta I(Z; S, A),
\end{align}
where we use the chain rule of mutual information to expand $I(Z; S, A) = I(Z; S) + I(Z; A | S)$. To optimize this mutual information objective, we derive a tractable objective, detailed in Appendix~\ref{sec:app_p2}. 

\begin{proposition}[State Belief Gaussian Fitting based on VIB]\label{proposition2}
Let $Z$ be a stochastic latent representation of a state-action pair $(S, A)$ generated by an encoder $p_\vartheta(z|s, a)$. Let $q_\psi(s'|z)$ be a variational decoder, $p_\varphi(z|s)$ be a state belief, and $p(z)$ be a fixed marginal prior (e.g., $\mathcal{N}(0, I)$). The VIB objective $\max  I(Z; S') - \beta I(Z; S, A)$ can be optimized via the following approximation objective:
\begin{align}\label{eq:sbvib}
    \mathcal{L}_{\mathrm{SBGF}}(\psi, \vartheta, \varphi) = \mathbb{E}_{(s,a,s') \sim \mathcal{D}} \Big[ &\underbrace{\mathbb{E}_{z \sim p_\vartheta(\cdot|s,a)} [-\log q_\psi(s' | z)]}_{\text{Reconstruction}} \nonumber \\
    &+ \beta \Big( \underbrace{D_{\mathrm{KL}}(p_\vartheta(z|s,a) \| p_\varphi(z|s))}_{\text{Bidirectional Action IB}} + \underbrace{D_{\mathrm{KL}}(p_\varphi(z|s) \| p(z))}_{\text{State IB}} \Big) \Big]. 
\end{align}
\end{proposition}
To estimate the latent divergence in Eq.~\eqref{eq:surrogate}, the encoder $p_\vartheta(z|s,a)$ and decoder $q_\psi(s'|z)$ model environment dynamics, yielding the target policy's distribution by taking the expectation over actions: $\hat p_{\pi}(z|s)=\mathbb{E}_{a\sim\pi(\cdot|s)} p_\vartheta(z|s,a)$. Meanwhile, the state belief encoder $p_\varphi(z|s)$ directly captures the behavior policy's consequences, providing $\hat p_{\pi_\beta}(z|s)= p_{\varphi}(z|s)$.

The core innovation of our SBGF lies in the \textit{Bidirectional Action IB} term. Directly integrating the encoder over the actions of the behavior policy would yield a complex Gaussian Mixture Model, making the analytical calculation of latent divergence in Eq.~\eqref{eq:surrogate} intractable. To circumvent this, the bidirectional constraint fits the aggregated behavioral representations into a Gaussian state belief ($p_\varphi$), which simultaneously serves to regularize the encoder.

Through this representation learning framework, we establish the structured Gaussian latent space that enables the analytical computation of the surrogate divergence. However, a critical question naturally arises: \textit{when using this surrogate divergence to approximate the true consequence divergence, how large is the resulting discrepancy?} We formally analyze this gap in the following section.

\subsection{Theoretical Guarantee: Bounding the Surrogate Discrepancy} \label{sec:theory_bound}
To address the question raised above, we now formalize a theoretical upper bound on the gap between the true consequence divergence and our surrogate latent divergence estimated from the dataset.

\begin{theorem}[Surrogate Discrepancy Bound] \label{thm:surrogate_bound}
Let $\hat p_{\pi}(z|s)=\mathbb{E}_{a\sim\pi(\cdot|s)} p_\vartheta(z|s,a)$ for any given policy $\pi$, and let $\hat p_{\pi_\beta}(z|s)= p_{\varphi}(z|s)$ be the models estimated by SBGF in Eq.~\eqref{eq:sbvib} on an offline dataset $\mathcal{D}$ of size $N$. With probability at least $1 - \delta$, the following bound holds for the discrepancy between the true consequence divergence and the empirical surrogate divergence:
\begin{small}
\begin{align}
    &\left| \mathbb{E}_{s \sim d^{\pi_\beta}} \big[ D_{\mathrm{KL}}(P_{\pi}(s'|s) \| P_{\pi_\beta}(s'|s)) \big] - \mathbb{E}_{s \sim \mathcal{D}} \big[ D_{\mathrm{KL}}(\hat p_{\pi}(z|s)  \| \hat p_{\pi_\beta}(z|s)) \big] \right| \label{eq:theorem1} \\
    \le &\underbrace{\mathbb{E}_{s,s'} \left[ D_{\mathrm{KL}}(p_{\pi}(z|s,s') \| p_{\pi_\beta}(z|s,s')) \right]}_{\substack{\text{\textbf{(I) Dynamics Info Compression (Bias)}} \\ \text{\textbf{(Increases as }} \beta ~\uparrow \text{\textbf{)}}}}  
    + \underbrace{\mathcal{O}\left( \sqrt{\frac{\mathrm{Var}_s(D_{\mathrm{KL}}(p_{\pi}(z|s)  \| p_{\pi_\beta}(z|s)) \log(1/\delta)}{N}} \right)}_{\substack{\text{\textbf{(II) Estimation Error (Variance)}} \\ \text{\textbf{(Decreases as }} \beta ~\uparrow \text{\textbf{)}}}} 
    + \underbrace{C_\pi \epsilon_{\mathrm{opt},\beta}.}_{\text{\textbf{(III) Policy Shift}}} \nonumber
\end{align}
\end{small}
where the policy ratio in the latent space is bounded by $C_\pi = \max_{s \in \mathcal{D}, z} \frac{p_\pi(z|s)}{p_{\pi_\beta}(z|s)}$, and $\epsilon_{\mathrm{opt},\beta}$ is an error related to the bidirectional action IB loss in Eq.~\eqref{eq:sbvib} and fundamental approximation limit of fitting a single Gaussian to a potentially multimodal behavior-induced latent distribution. 
\end{theorem}

\textbf{Remark.} 
Theorem~\ref{thm:surrogate_bound} reveals that our surrogate approximation is systematically governed by a fundamental Bias-Variance trade-off dictated by the information compression parameter $\beta$. \textit{(Note that an excessively large $\beta$ forces the latent $z$ to collapse to the prior, losing all dynamics prediction capability. Thus, the meaningful trade-off occurs in the regime of relatively small $\beta$.)} The detailed proof can be found in Appendix~\ref{sec:app_t1}.

\textbf{Term (I) acts as the Bias.} When $\beta$ is small, the mapping $z \to s'$ is near-deterministic. Consequently, $p_\pi(z|s,s')$ is largely policy-independent, keeping this bias low. As $\beta$ increases, compression maps multiple $z$ to the same $s'$, exposing policy differences and increasing the bias.

\textbf{Term (II) acts as the Variance.} 
It captures the cross-state variance of the latent divergence defined in Eq.~\eqref{eq:surrogate}. Increasing $\beta$ acts as a strong regularizer, smoothing out action details to effectively suppress the fluctuation of $D_{\text{KL}}(p_\pi(z|s) \| p_{\pi_\beta}(z|s))$.


\textbf{Term (III) highlights the conservatism.} The policy ratio $C_\pi$ indicates that this surrogate becomes ineffective when evaluating OOD actions, directly motivating the necessity of policy constraint.

This theoretical insight provides a principled \textit{guideline for empirical $\beta$ tuning}: a larger $\beta$ mitigates variance in noisy or stochastic environments, whereas a smaller $\beta$ preserves dynamics and minimizes bias in near-deterministic settings. Ultimately, Theorem~\ref{thm:surrogate_bound} mathematically justifies the SBGF design, establishing a tractable surrogate for the policy optimization detailed in the next section.

\begin{figure}
  \centering
  \includegraphics[width=0.8\columnwidth]{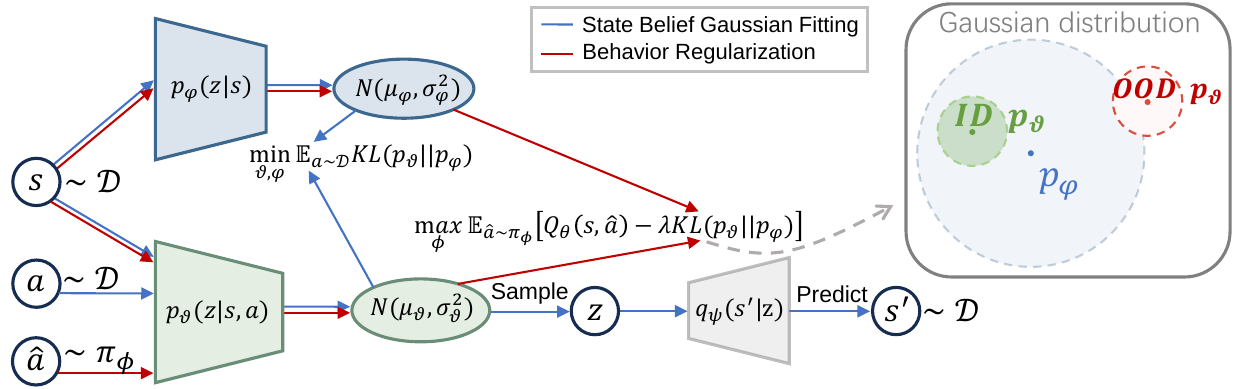}
\caption{The architecture of the proposed StBR framework.The blue flow represents the learning of latent representations (Eq.~\eqref{eq:sbvib}) through the proposed State Belief Gaussian Fitting. The red flow highlights the application of the surrogate latent divergence for behavior regularization in Eq.~\eqref{eq:pi_practical}.}
  \label{fig:sbr_frame}
\end{figure}

\subsection{Implementation: Structural Behavior Regularization}\label{sec:3.4}
The surrogate divergence derived above can be plugged into standard off-policy algorithms with minimal modifications. We choose TD3~\cite{td3bc,awac,spot} as our base RL algorithm to verify the effectiveness of the proposed StBR framework.

\textbf{Point-wise Constraint.} While the surrogate divergence can be directly computed when $\pi_\phi$ is parameterized as a simple Gaussian or deterministic policy, it becomes intractable for more expressive generative model policies (e.g., flow policies~\citep{fql}). A practically feasible alternative is to instead minimize its point-wise upper bound:
\begin{equation}
    D_{\mathrm{KL}}(\hat{p}_{\pi_\phi}(z|s) \| \hat{p}_{\pi_\beta}(z|s)) \le \mathbb{E}_{a\sim{\pi_\phi}(\cdot|s)} \big[D_{\mathrm{KL}}(p_\vartheta(z|s, a) \| p_\varphi(z|s))\big].
\end{equation}
Beyond computational efficiency, this surrogate objective avoids a risk-dilution issue where averaging over multiple samples can mask OOD actions. The proof sketch of Theorem~\ref{thm:surrogate_bound} in Appendix \ref{sec:app_t1s} details this mechanism.

\textbf{Closed-form Gaussian Latent Divergence.} Leveraging the Gaussian models $p_\vartheta(z|s,a) = \mathcal{N}(\mu_\vartheta, \sigma_\vartheta|s,a)$ and $p_\varphi(z|s) = \mathcal{N}(\mu_\varphi, \sigma_\varphi|s)$ learned via SBGF in Eq.~\eqref{eq:sbvib}, the point-wise divergence can be computed analytically without sampling:
\begin{equation}
D_{\mathrm{KL}}(p_\vartheta(z|s, a) \| p_\varphi(z|s)) = -\frac{1}{2} \Big( 1 + \log\frac{\sigma_{\vartheta}^2}{\sigma_{\varphi}^2} - \frac{\sigma_{\vartheta}^2+(\mu_{\vartheta} - \mu_{\varphi})^2}{\sigma_{\varphi}^2} \Big) .
\end{equation}

\textbf{Policy Optimization.} 
Enforcing our constraint is computationally elegant, as it simply requires aligning the representations from two encoders. Fusing this lightweight latent divergence constraint with the standard Q-value objective, we update the policy $\pi_\phi$ by:
\begin{equation} \label{eq:pi_practical}
\mathcal{L}_\pi(\phi) = \mathbb{E}_{s \sim \mathcal{D},a \sim \pi_\phi(\cdot|s)} \big[ -Q_\theta(s, a) + \lambda D_{\mathrm{KL}}(p_\vartheta(z|s, a) \| p_\varphi(z|s)) \big],
\end{equation}
where $\lambda$ is a hyperparameter controlling the regularization strength.

\textbf{Overall algorithm.} StBR follows a standard actor-critic paradigm and optimizes the SBGF via Eq.~\eqref{eq:sbvib} right after the value and policy updates. Figure~\ref{fig:sbr_frame} gives the architecture of the proposed StBR framework, while the whole training process is shown in Appendix~\ref{sec:app_main}.

\section{Related Works}

\textbf{Behavior regularized offline RL.}
Behavior regularization has evolved from early MSE cloning to VAE-based divergence matching, and more recently toward high-expressivity diffusion and flow matching models~\citep{td3bc,cql,spot,cped,dql,fql}. These advancements focus on achieving increasingly precise modeling of the underlying behavior policy distribution to capture its full complexity.
However, such action-distribution matching paradigms remain agnostic to environmental dynamics, where minor action deviations can trigger significant state drift due to a lack of dynamic consistency.


\textbf{Consequence-aware offline RL.}
Several offline RL approaches can be viewed as consequence-aware by anticipating future states. For instance, POR~\citep{por} employs hierarchical state-planning with an inverse model, while ODAC~\citep{odac} constrains actions using explicit dynamics models and value uncertainty. Alternatively, out-of-sample state recovery methods (SDC~\citep{sdc} and OSR~\citep{osr}) utilize inverse dynamics or dynamics ensembles to regularize the policy via transition alignment, ensuring agents can recover to the data manifold, thereby enhancing generalization. Unlike these methods, which struggle in high-dimensional environments due to their reliance on computationally heavy explicit state predictions or ensemble models, our StBR imposes a lightweight consequence-aware policy constraint directly via a latent divergence.




\textbf{Other competitive offline RL methods.} 
Besides density matching, more types of competitive offline RL methods exist. 
Conservative Q-Learning (CQL)~\citep{cql} implicitly constrains action distributions via Q-value penalization. PBRL~\citep{pbrl} enforces pessimism using ensemble variance; however, alongside high computational costs, this uncertainty becomes uninformative in sparse-reward environments due to lack of reward signals. Conversely, IQL~\cite{iql} eliminates OOD queries via in-sample expectile regression, but its strict reliance on existing data limits generalization. We provide a broader discussion of additional related works in Appendix~\ref{sec:appp_related}.





\begin{table*}[t]
\centering
\caption{The mean and standard deviation of normalized scores after 1M steps of offline RL training, averaged over 5 seeds and the final 10 evaluations (the final 4 for AntMaze). The highest score in each task is bolded, and the second-best is underlined.}\label{tab:main_results}
\resizebox{\linewidth}{!}{
\begin{tabular}{ll|llllllllll|l}
\toprule
                             & Task Name      & BC        & CQL       & TD3BC     & IQL       & PBRL      & POR       & SDC       & SPOT      & FQL       & ENSP      & StBR (Ours) \\ \midrule
\multirow{3}{*}{halfcheetah} & medium         & 41.2±0.4  & 46.9±0.4  & 48.2±0.1  & 47.1±0.1  & 57.9±1.5  & 48.8±0.5  & 45.9±0.5  & \underline{58.4}±1.0  & 54.4±2.2  & 55.4±0.6  & \textbf{60.7}±0.7    \\
                             & medium-replay  & 32.9±3.4  & 45.3±0.3  & 44.6±0.1  & 44.3±0.3  & 45.1±8.0  & 43.5±0.9  & 45.1±0.5  & 52.2±1.2  & \underline{53.2}±0.6  & 43.4±1.5  & \textbf{54.4}±1.0    \\
                             & medium-expert  & 50.4±4.0  & 95.0±1.4  & 91.7±1.9  & 89.2±1.2  & 92.3±1.1  & 94.7±2.2  & 76.3±5.2  & 86.9±4.3  & \underline{95.3}±1.3  & 88.1±4.9  & \textbf{97.8}±0.2    \\ \hline
\multirow{3}{*}{hopper}      & medium         & 46.1±1.8  & 61.9±6.4  & 59.0±0.7  & 63.6±2.9  & 75.3±31.2 & 78.6±7.2  & 64.7±3.5  & \underline{86.0}±8.7  & 77.1±10.3 & 80.4±14.7 & \textbf{100.9}±1.0   \\
                             & medium-replay  & 22.4±5.9  & 86.3±7.3  & 65.4±14.0 & 95.4±4.9  & \underline{100.6}±1.0 & 98.9±2.1  & 94.8±6.5  & 100.2±1.0 & 79.5±4.5  & 91.0±12.9 & \textbf{100.8}±1.9   \\
                             & medium-expert  & 62.4±13.9 & 96.9±15.1 & 101.6±2.0 & 101.6±7.4 & \underline{110.8}±0.8 & 90.0±12.1 & 99.9±8.5  & 99.3±7.1  & 103.5±6.7 & 65.5±12.5 & \textbf{111.6}±0.9   \\ \hline
\multirow{3}{*}{walker2d}    & medium         & 62.0±7.7  & 79.5±3.2  & 83.8±0.3  & 80.2±2.8  & \textbf{89.6}±0.7  & 81.1±2.3  & 82.7±1.9  & 86.4±2.7  & 81.9±3.4  & 80.4±3.8  & \underline{87.2}±0.9    \\
                             & medium-replay  & 26.4±8,4  & 76.8±10.0 & 80.3±1.8  & 69.9±8.4  & 77.7±14.5 & 76.6±6.9  & 78.5±6.0  & \textbf{91.6}±2.8  & 76.4±5.6  & 78.5±6.9  & \underline{86.3}±1.5    \\
                             & medium-expert  & 81.1±12.5 & 109.1±0.2 & 110.1±0.3 & 109.6±7.4 & 110.1±0.3 & 109.1±0.7 & 109.2±1.4 & \textbf{112.0}±0.5 & \underline{111.3}±0.3 & 109.4±0.8 & \underline{111.3}±1.4   \\ \midrule
\multicolumn{2}{l}{MuJoCo Total}              & 424.9     & 697.7     & 684.7     & 700.9     & 759.4     & 721.3     & 697.1     & \underline{773.0}     & 732.5     & 707.1     & \textbf{811.0}       \\ \midrule
\multirow{6}{*}{antmaze}     & umaze          & 58.8±5.1  & 92.8±1.9  & 70.8±39.2 & 89.6±4.2  & 0.0±0.0   & 90.6±7.1  & 81.4±3.8  & 93.5±2.4  & \textbf{96.0}±2.0  & \underline{94.3}±3.0  & \textbf{96.0}±2.3    \\
                             & umaze-diverse  & 62.0±5.9  & 37.3±3.7  & 44.8±11.6 & 65.6±8.3  & 0.0±0.0   & 71.3±12.1 & 49.6±10.4 & 40.7±5.1  & \textbf{89.0}±5.0  & 69.1±10.4 & \underline{83.4}±7.2    \\
                             & medium-play    & 0.0±0.0   & 65.8±11.6 & 0.3±0.4   & 76.4±2.7  & 0.0±0.0   & \underline{84.6}±5.6  & 55.0±9.6  & 74.7±4.6  & 78.0±7.0  & 75.4±8.5  & \textbf{94.2}±1.9    \\
                             & medium-diverse & 0.2±0.0   & 67.3±3.6  & 0.3±0.4   & 72.8±7.0  & 0.0±0.0   & \underline{79.2}±3.1  & 56.6±10.3 & 79.1±5.6  & 71.0±13.0 & 73.4±8.0  & \textbf{90.6}±3.6    \\
                             & large-play     & 0.0±0.0   & 20.8±7.3  & 0.0±0.0   & 42.0±3.8  & 0.0±0.0   & 58.0±12.4 & 20.8±8.0  & 35.3±8.3  & \underline{84.0}±7.0  & 55.4±4.8  & \textbf{84.2}±2.7    \\
                             & large-diverse  & 0.0±0.0   & 20.5±13.2 & 0.0±0.0   & 46.0±4.5  & 0.0±0.0   & 73.4±8.5  & 25.8±7.5  & 36.3±13.7 & \underline{83.0}±4.0  & 59.0±9.8  & \textbf{85.8}±5.1    \\ \midrule
\multicolumn{2}{l}{AntMaze Total}             & 121.0     & 304.5     & 116.2     & 392.4     & 0.0       & 457.1     & 289.2     & 359.6     & \underline{501.0}     & 426.6     & \textbf{534.2}       \\ \midrule
pen                          & human          & 70.2±9.4  & 13.7±17.0 & -3.9±0.2  & \underline{78.5}±8.2  & 35.4±3.3  & 76.9±5.1  & 64.5±7.5  & 47.5±7.4  & 53.0±6.0  & -3.1±0.0  & \textbf{99.5}±3.4    \\
hammer                       & human          & 1.2±0.9   & 0.1±0.1   & -0.3±0.1  & 1.8±0.8   & 0.4±0.3   & \underline{1.9}±0.7   & 0.7±0.5   & 1.2±0.4   & 1.0±1.0   & 0.3±0.1   & \textbf{3.1}±0.6     \\
door                         & human          & 1.4±1.8   & \underline{5.5}±1.3   & 1.0±0.0   & 3.3±1.8   & 0.1±0.0   & 3.0±1.2   & 0.0±0.0   & 0.0±0.0   & 0.0±0.0   & -0.2±0.1  & \textbf{7.9}±2.0     \\ \midrule
\multicolumn{2}{l}{Adroit Total}              & 72.8      & 19.3      & -3.2      & \underline{83.6}      & 35.9      & 81.8      & 65.2      & 48.7      & 54.0      & -3.0      & \textbf{110.5}       \\ \midrule
\multicolumn{2}{l}{Wall time(min)}            & 33        & 403       & 78        & 108       & 444       & 147       & 417       & 111       & 110       & 108       & 114         \\ 
\bottomrule
\end{tabular}
}
\end{table*}

\section{Experiments}
\subsection{Performance Comparison}
\textbf{Tasks.} 
We evaluate StBR across three domains: continuous-control MuJoCo locomotion, sparse-reward AntMaze, and high-dimensional Adroit human manipulation.

\textbf{Hyperparameter tuning.} 
The coefficient $\lambda$ in Eq.~\eqref{eq:pi_practical} is essential for behavior regularization. Following prior works~\citep{bear,brac,spot}, we allow limited access to the online environment to tune this hyperparameter across a minimal candidate set of $\{0.03, 0.05, 0.1\}$ (only 3 choices), which is a reasonable setup for practical applications. See Appendix~\ref{sec:app_main} for additional discussion and details.

\textbf{Main Results.}
Table~\ref{tab:main_results} shows that StBR outperforms existing baselines across diverse benchmarks, underscoring two key advantages. 
First, unlike action-distribution-based methods (e.g., CQL, SPOT, FQL) that suffer from consequence-shift vulnerabilities, StBR effectively facilitates trajectory stitching in long-horizon, sparse-reward tasks like AntMaze. While specific baselines perform selectively, such as FQL's flow modeling succeeding in AntMaze but turning overly conservative in dense-reward MuJoCo, or PBRL collapsing in sparse rewards due to ungrounded Q-variance guidance, StBR delivers robust performance across both regimes.
Second, StBR circumvents the dynamics modeling errors inherent to explicit consequence-based approaches (e.g., POR, SDC, and an explicit next-state prediction baseline, ENSP). Although explicit next-state modeling can aid stitching in low-dimensional tasks, it scales poorly to high-dimensional state spaces like Adroit due to the curse of dimensionality. For instance, ENSP achieves competitive scores in low-dimensional MuJoCo but exhibits severe variance, trails implicit planning in AntMaze, and fails catastrophically in Adroit (45-dim). By enforcing consequence constraints implicitly, StBR avoids intractable dynamics modeling while retaining consequence conservatism, achieving state-of-the-art performance throughout.

\textbf{Runtime.} 
We evaluated the runtime (in minutes) on the \texttt{hopper-medium}. As shown in Table~\ref{tab:main_results}, StBR incurs a computational overhead comparable to that of most behavior regularization methods, demonstrating the high efficiency of our implementation.

\subsection{Hyperparameter Analysis of StBR}\label{sec:sensitive}
\textbf{Effect of $\lambda$.} 
The Lagrangian multiplier $\lambda$ in Eq.~\eqref{eq:pi_practical} is essential in regularization methods, as it controls the balance between policy improvement and conservatism. However, exhaustive hyperparameter tuning is prohibitively costly in offline settings. Consequently, existing methods~\citep{td3bc,spot,fql} typically restrict tuning to a minimal candidate set. However, when deploying algorithms on new tasks in practice, the ability to quickly find an effective $\lambda$ varies across different methods, which directly reflects their ease of use. 

To evaluate method usability, we sweep $\lambda$ from 0.015 (loose conservatism) to 0.5 (over-conservatism) to demonstrate which method can reliably achieve superior performance under a standardized tuning process. Note that for a fair comparison, we apply a Q-value normalization trick to align the objective scales across all baselines (more implementation details are provided in Appendix~\ref{sec:app_main}). 

Figure~\ref{fig:exp1} (a-b) presents the results on the \texttt{hopper-medium-v2} and \texttt{antmaze-large-diverse-v2} tasks, two representative tasks from the D4RL benchmark widely known for their high sensitivity to hyperparameters due to a large proportion of suboptimal data. When evaluating performance across varying values of $\lambda$, StBR consistently outperforms all baseline methods. This demonstrates that StBR is highly robust to parameter choices, making it significantly easier to achieve superior performance during practical deployment.

\textbf{SBGF-related hyperparameters.}
First, the information compression parameter $\beta$ controls the information bottleneck constraint: a larger $\beta$ forces the latent representation to be closer to a standard Gaussian distribution, resulting in a highly compressed latent space. Consequently, the numerical magnitude of the deviation within this space shrinks. To compensate for this reduced scale and maintain an effective behavior regularization penalty, a larger weight $\lambda$ is required. Figure~\ref{fig:exp1} (c) demonstrates this dynamic coupling: as $\lambda$ increases, the optimal $\beta$ (marked with a star) consistently shifts towards larger values. Notably, StBR achieves competitive performance across these configurations. Furthermore, Theorem~\ref{thm:surrogate_bound} suggests that the optimal $\beta$ is also inherently task-dependent, a property we further analyze in Section~\ref{sec:theo_verify}.

Next, we examine the sensitivity of StBR to the latent representation capacity by varying its dimension $|Z| \in [0.5,2] \times |S|$, as illustrated in Figure~\ref{fig:exp1} (d). StBR demonstrates strong robustness across a broad intermediate range (from 1 to 1.5 $\times|S|$). Performance degradation is primarily observed at the extremes: an excessively low dimension severely bottlenecks information, preventing the model from capturing necessary environmental dynamics, whereas an overly high dimension introduces representation redundancy and optimization challenges. 
Overall, the empirical results indicate that StBR is highly robust to SBGF-related hyperparameters within reasonable configurations.

\begin{figure}[t]
  \centering
  \includegraphics[width=\columnwidth]{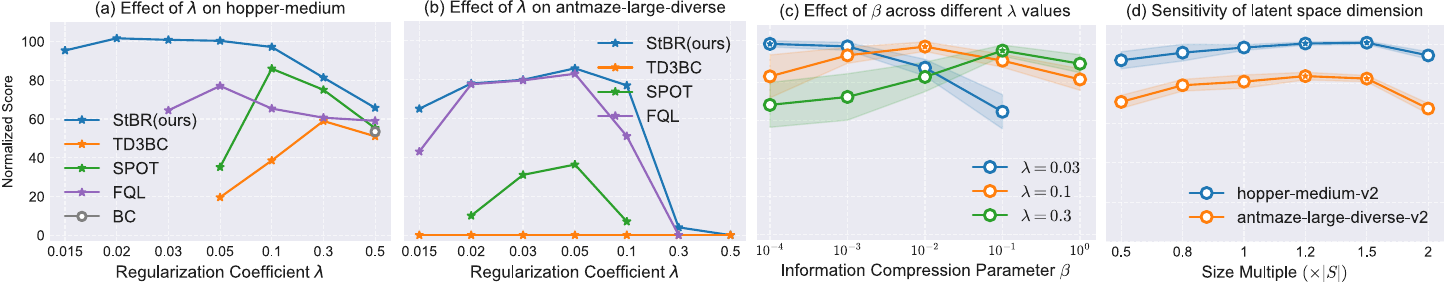} 
  \caption{Hyperparameter sensitivity analysis. (a-b) When evaluating the performance at varying levels of constraint strength ($\lambda$) on two representative tasks, StBR takes shape of the "upper envelope" of all methods, showing that StBR is highly robust to parameter choices, making it significantly easier to achieve superior performance during practical deployment. (c-d) Sensitivity analysis of SBGF-related hyperparameters regarding $\beta$ and latent space dimension.}
  \label{fig:exp1} 
\end{figure}

\begin{figure}[t]
  \centering
  \includegraphics[width=\columnwidth]{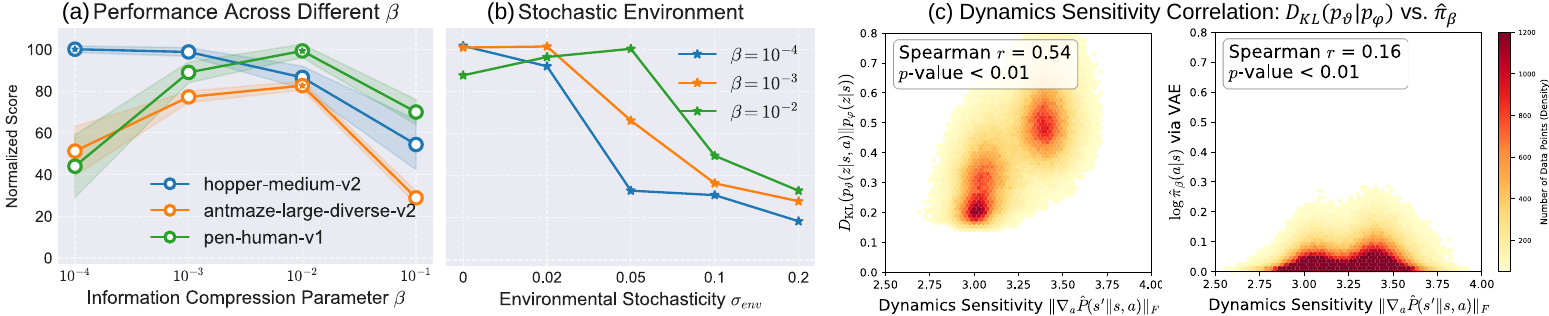} 
  \caption{Theoretical validation experiments. (a-b) Demonstrating that a larger $\beta$ should be chosen in high-dimensional tasks and highly stochastic environments to maintain superior performance. (c) Correlation analysis between dynamic sensitivity and different metrics: the proposed consequence-aware divergence and the action-distribution measure. Results validate that latent divergence (left, Spearman $r=0.54$) exhibits a significantly stronger correlation with dynamics sensitivity than the VAE-based density measure (right, $r=0.16$).
  }
  \label{fig:exp2} 
\end{figure}

\subsection{Theoretical Validations}\label{sec:theo_verify}

\textbf{Validating the Bias-Variance Trade-off ($\beta$).} 
Theorem 1 suggests that $\beta$ should be appropriately larger in high-dimensional tasks and highly stochastic environments. To empirically validate this theoretical insight, we evaluate StBR under two distinct regimes:

(a) \textit{Task dimensions:} We compare the performance of StBR across tasks with different state dimensions by varying $\beta \in \{10^{-4}, 10^{-3}, 10^{-2}, 10^{-1}\}$. As illustrated in Figure~\ref{fig:exp2} (a), for the low-dimensional \texttt{hopper} task, the optimal $\beta$ is $10^{-4}$. In contrast, for the high-dimensional \texttt{antmaze} and \texttt{pen} tasks, the optimal $\beta$ shifts to $10^{-2}$.

(b) \textit{Environment stochasticity:} We evaluate StBR on the \texttt{hopper-medium} task with artificially injected action execution stochasticity (i.e., simulated motor errors). As shown in Figure~\ref{fig:exp2} (b), as the variance of the injected noise ($\sigma_{env}^2$) increases ( higher environmental stochasticity), the optimal value of $\beta$ correspondingly increases to maintain robust performance.

\textbf{Correlation with Dynamics Sensitivity.} 
A core motivation of StBR is to adaptively tighten behavior constraints in regions where transitions are highly sensitive to action deviations. To validate this, we estimate dynamic sensitivity using the Frobenius norm of the Jacobian matrix of a pre-trained dynamics model $\|\nabla_a \hat{P}(s'|s, a)\|_F$. Figure~\ref{fig:exp2}(c) illustrates the relationship between different metrics and dynamics sensitivity for state-action pairs from the \texttt{hopper-medium} dataset. The proposed consequence-aware divergence demonstrates a strong positive correlation, yielding a Spearman coefficient of $0.54$ ($p < 0.01$). In contrast, the action-distribution measure via VAE~\citep{spot}, $\hat{\pi}_\beta(a|s)$, exhibits a significantly weaker relationship, with a correlation of only $0.16$ ($p < 0.01$). These results confirm that latent divergence effectively serves as a consequence-aware surrogate, naturally imposing stricter regularization in dynamics-sensitive regions.


\subsection{Ablation Study}\label{sec:ablation}
\begin{multicols}{2}
\textbf{Bidirectional Action IB.} 
First, we ablate the bidirectional constraint in Eq.~\eqref{eq:sbvib}. Results are presented in Figure~\ref{fig:exp3} (left). As expected, removing this constraint causes our objective to degenerate into a standard VIB. Under this formulation, the state belief completely collapses into a standard Gaussian prior, thereby losing the critical semantic information regarding the high-probability consequences induced by the behavior policy. Consequently, the variant suffers catastrophic performance degradation, even entirely ineffective on high-dimensional environments such as the AntMaze and Adroit tasks.
\begin{figure}[H]
    \centering
    \includegraphics[width=1.0\linewidth]{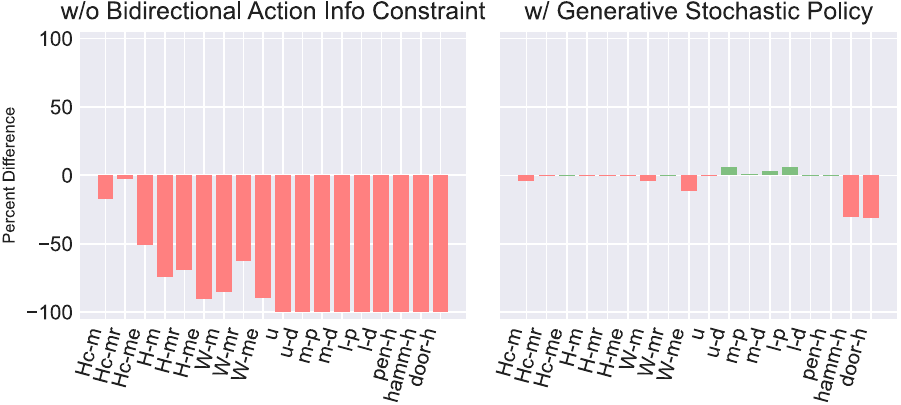} 
    \caption{Percent performance difference of StBR variants compared to the original algorithm.}
    \label{fig:exp3} %
\end{figure}
\end{multicols}
\textbf{Are stochastic policies better than deterministic policies?} 
Standard Gaussian stochastic policies are often disadvantageous in offline policy constraint methods due to the mode-averaging problem~\citep{spot}. 
To investigate the performance of StBR when learning a stochastic policy, we train a one-step generative policy based on the FQL~\citep{fql} framework. 
As depicted in Figure~\ref{fig:exp3} (right), the results demonstrate that the generative stochastic policy primarily exhibits advantages in sparse-reward tasks (AntMaze). Conversely, in tasks where the optimal actions are near-deterministic, the stochastic policy generally underperforms the deterministic one. This difference arises because stochastic policies assign non-zero probability mass to sub-optimal actions. 
Extensive supplementary experiments further validating the capabilities of StBR are detailed in Appendix~\ref{sec:app_experiments}.

\section{Conclusion}
This paper proposes the Structural Behavior Regularization (StBR) framework, reformulating behavior regularization as a consequence-aware constraint process to solve the consequence-shift vulnerability problem in offline RL. The proposed surrogate latent divergence avoids dynamic density estimation via structured representation learning and is theoretically proven to be a reliable approximation. Experiments show StBR achieves superior performance across diverse offline RL benchmarks.

\textbf{Limitations.}
While StBR achieves strong empirical performance, it comes with several limitations that suggest promising directions for future research.
First, StBR imposes conservatism relative to the behavior policy, meaning that severely sub-optimal behavioral data may induce an inherent tension between conservatism and policy improvement; developing adaptive or evolving regularization mechanisms could mitigate this conflict.
Second, the unimodal Gaussian latent assumption embodies a fundamental trade-off between expressiveness and tractability. Due to its light-tailed nature and propensity for mode-averaging, a single Gaussian struggles in regimes with severe multimodality or sharp transition boundaries—such as environments with heavy-tailed stochasticity where rare transitions govern task success (as observed in our stochastic environment), or systems with contact discontinuities and bifurcations where unimodal priors induce artificial smoothing. Extending our latent representation to expressive multimodal priors (e.g., mixtures or diffusion models) is a natural avenue to capture such dynamics.
Finally, following standard practice in current offline RL literature, our setup relies on limited online evaluations for hyperparameter tuning due to the open challenges in offline policy evaluation. Nevertheless, this introduces non-trivial deployment overhead and safety concerns in high-stakes domains where online interactions are strictly prohibited. Developing reliable, strictly offline model selection and evaluation criteria remains an essential prerequisite for fully autonomous deployment.

\section*{Acknowledgements}
This work is partially supported by National Natural Science Foundation of China (62476128).





\newpage
\bibliographystyle{named}
\bibliography{nips26}

\appendix

\newpage

\section{Related Work}\label{sec:appp_related}
\paragraph{Behavior-regularized methods.} 
To mitigate the severe extrapolation errors caused by out-of-distribution (OOD) actions, the dominant paradigm in offline RL is behavior regularization, which penalizes the learned policy for deviating from the offline dataset. Early methods enforce this via explicit architectural constraints (e.g., BCQ~\citep{bcq}, PLAS~\citep{plas}) or by minimizing divergence metrics like MSE~\citep{td3bc}, MMD~\citep{bear}, and KL-divergence~\citep{brac}. To address the limited expressivity of these early models on complex datasets, recent advancements leverage powerful generative models such as VAEs~\citep{spot}, diffusion models~\citep{eql}, and flow matching~\citep{fql} to explicitly match the action distribution of the behavior policy. Despite their success, these methods are fundamentally action-distribution centric and inherently agnostic to the underlying dynamics structure. Due to function approximation and finite data, minor action deviations are inevitable. In dynamics-sensitive regions, such minor action deviations can compound into drastically different state transitions, inadvertently driving the agent toward unsafe OOD regions. In contrast, StBR shifts this paradigm by imposing policy constraints based on the induced state-consequences rather than pure action-distribution matching.

\paragraph{Consequence-aware methods.}
Several offline RL approaches can be viewed as consequence-aware by anticipating future states. For instance, POR~\citep{por} employs hierarchical state-planning with an inverse model, while ODAC~\citep{odac} and LUC~\citep{luc} constrain actions using explicit dynamics models and value uncertainty. Out-of-sample state recovery methods, such as OSR~\citep{osr} and SDC~\citep{sdc}, utilize inverse dynamics or ensemble models to pull the agent back to the data manifold via transition alignment. While effective in simple environments, these methods rely heavily on explicit state predictions or computationally burdensome ensemble models, making them intractable in high-dimensional continuous control tasks. To the best of our knowledge, StBR is the first consequence-aware method that eliminates the need for explicit future state generation, dynamics density estimation, and ensemble models. StBR elegantly translates the consequence penalty into a lightweight latent divergence, making it highly computationally efficient while retaining the full benefits of forward-looking constraints.



\paragraph{Model-based methods.} 
Another line of research mitigating distribution shift operates by explicitly learning the environment dynamics. Model-based offline RL methods, such as MOPO~\citep{mopo} and MOReL~\citep{morel}, construct forward transition models and penalize policies for visiting states with high model uncertainty. However, such uncertainty estimation heavily relies on heuristics and is highly sensitive to hyperparameters due to the varying scales of uncertainty across states. RAMBO~\citep{rambo} further addresses this by replacing explicit uncertainty penalties with an adversarial framework, which directly modifies the transition dynamics to enforce pessimism. MOBILE~\citep{mobile} further addresses the inconsistency between the learned model and the Bellman operator. 
The fundamental limitation of these methods lies in their vulnerability to model errors. Accurately learning the underlying transition dynamics from static datasets is notoriously difficult, and inaccuracies in the forward model inevitably lead to severe compounding errors during multi-step rollouts. In contrast, StBR circumvents this by completely omitting explicit state-reconstruction decoders during the policy update phase, while still maintaining dynamics awareness through latent divergence regularization.

\paragraph{Dynamics-based representation learning.} 
Representation learning in offline RL often incorporates dynamics modeling. A common intuition is that learning a superior latent state representation can implicitly capture the system's evolution. Representative methods like TSRL~\citep{tsrl} and DASP~\citep{dasp} perform dynamic predictions within the latent space to mitigate over-conservatism and improve generalization. However, these approaches typically conduct the core RL optimization (e.g., value evaluation and policy optimization) entirely within the learned latent space. In contrast, StBR constructs the latent space solely to evaluate the consequence-aware penalty, while the actual RL learning remains grounded in the original observation space. This design makes StBR compatible with them and can be integrated with these approaches to pursue even stronger generalization performance.

\newpage

\section{Proofs}\label{sec:proofs}
This section provides the proofs for the propositions and theories in the paper.

\subsection{Proof of Proposition 1}\label{sec:app_p1}

\textbf{Proposition 1} \textit{(Surrogate Consequence Divergence).} \textit{Assume $z$ is a latent variable acting as a sufficient statistic for the state-action pair $(s, a)$ to predict the next state $s'$, forming the causal diagram $(s, a) \to z \to s'$. For any policy $\pi$ and behavior policy $\pi_\beta$, the KL divergence between their consequence state distributions is upper-bounded by the KL divergence between their corresponding latent distributions:}
\begin{align}
    D_{\mathrm{KL}} \big( P_{\pi}(s'|s) \| P_{\pi_\beta}(s'|s) \big) \le D_{\mathrm{KL}} \big( p_{\pi}(z|s) \| p_{\pi_\beta}(z|s) \big).
\end{align}

\begin{proof}
Given the causal diagram $(s, a) \to z \to s'$, the conditional independence $(s, a) \perp \! \! \! \perp s' | z$ naturally follows. Consequently, the transition dynamics can be decomposed via the latent variable $z$ as:
\begin{align}
    P(s'|s, a) = \int p(s'|z, s, a) p(z|s, a) \, dz = \int p(s'|z) p(z|s, a) \, dz.
\end{align}
Accordingly, we define the latent distribution under policy $\pi$ as $p_\pi(z|s) = \int p(z|s, a) \pi(a|s)da$. The transition dynamic then becomes $P_\pi(s'|s) = \int p(s'|z) p_\pi(z|s) \, dz$, where the decoder $p(s'|z)$ is an intrinsic property of the environment and remains independent of the policy.

To establish the bound, we consider the joint distribution $p_\pi(z, s' | s)$ and decompose its KL divergence. By definition, the state transition distribution marginalized over actions is $P_\pi(s'|s) = \sum_{a} P(s'|s, a) \pi(a|s)$. The complete error decomposition can be structured as follows:
\begin{align}
    &D_{\mathrm{KL}} \big( p_\pi(z|s) \| p_{\pi_\beta}(z|s) \big) \\
    \overset{(a)}{=} &D_{\mathrm{KL}} \big( p_\pi(z, s' | s) \| p_{\pi_\beta}(z, s' | s) \big) \\
    \overset{(b)}{=} &\mathbb{E}_{(z, s') \sim p_\pi} \left[ \log \frac{P_\pi(s'|s) p_\pi(z|s, s')}{P_{\pi_\beta}(s'|s) p_{\pi_\beta}(z|s, s')} \right] \label{eq:termb}\\
    \overset{(c)}{=} &D_{\mathrm{KL}} \big( P_\pi(s'|s) \| P_{\pi_\beta}(s'|s) \big) + \mathbb{E}_{s' \sim P_\pi(\cdot|s)} \left[ D_{\mathrm{KL}} \big( p_\pi(z|s, s') \| p_{\pi_\beta}(z|s, s') \big) \right] \label{eq:exact_gap}\\
    \overset{(d)}{\ge} &D_{\mathrm{KL}} \big( P_\pi(s'|s) \| P_{\pi_\beta}(s'|s) \big).
\end{align}

For (a): We construct the joint distribution conditioned on state $s$ as $p_\pi(z, s' | s) = p_\pi(z|s)p(s'|z)$. Expanding the joint KL divergence along the latent space $z$ yields:
\begin{align}
   &D_{\mathrm{KL}} \big( p_\pi(z, s' | s) \| p_{\pi_\beta}(z, s' | s) \big) \\
    =& \int \int p_\pi(z|s)p(s'|z) \left[ \log \frac{p_\pi(z|s)}{p_{\pi_\beta}(z|s)} + \log \frac{p(s'|z)}{p(s'|z)} \right] \, dz \, ds' \\
    =& \int p_\pi(z|s) \log \frac{p_\pi(z|s)}{p_{\pi_\beta}(z|s)} \underbrace{\left( \int p(s'|z) \, ds' \right)}_{=1} dz + \int p_\pi(z|s) \left( \int p(s'|z) \underbrace{\log 1}_{=0} \, ds' \right) dz \\
    =& D_{\mathrm{KL}} \big( p_\pi(z|s) \| p_{\pi_\beta}(z|s) \big).
\end{align}

For (b): We apply Bayes' rule to factorize the joint distribution over $(z, s')$ as: $p_\pi(z, s' | s) = P_\pi(s'|s) p_\pi(z | s, s')$. Here, $p_\pi(z | s, s') = \frac{p(s'|z)p_\pi(z|s)}{P_\pi(s'|s)}$ represents the posterior of the latent variable $z$ conditioned on the next state $s'$. Substituting this factorization into the joint KL divergence yields Eq.~\eqref{eq:termb}.

For (c): We decompose the integral into two parts. Marginalizing the first term over $z$ recovers the marginal KL divergence $D_{\mathrm{KL}}(P_\pi \| P_{\pi_\beta})$. The second term represents the expected KL divergence between the posteriors: $\mathbb{E}_{s' \sim P_\pi(\cdot|s)} \big[ \int p_\pi(z|s, s') \log \frac{p_\pi(z|s, s')}{p_{\pi_\beta}(z|s, s')}  dz \big]$.

For (d): By definition, the KL divergence is strictly non-negative. Consequently, the second term in Eq.~\eqref{eq:exact_gap} is also non-negative, which yields the final inequality bound.

This concludes the proof.

\textit{Remark:} 
Eq.~\eqref{eq:exact_gap} further reveals that the exact gap between the two KL divergences is the expected KL divergence of the posterior latent distributions.

\end{proof}

\subsection{Proof of Proposition 2}\label{sec:app_p2}

\textbf{Proposition 2} \textit{(State Belief Gaussian Fitting based on VIB).} \textit{Let $Z$ be a stochastic latent representation of a state-action pair $(S, A)$ generated by an encoder $p_\vartheta(z|s, a)$. Let $q_\psi(s'|z)$ be a variational decoder, $p_\varphi(z|s)$ be a state belief, and $p(z)$ be a fixed marginal prior (e.g., $\mathcal{N}(0, I)$). The VIB objective $\max  I(Z; S') - \beta I(Z; S, A)$ can be optimized via the following approximation objective:}
\begin{align}
    \mathcal{L}_{\mathrm{SBGF}}(\psi, \vartheta, \varphi) = \mathbb{E}_{(s,a,s') \sim \mathcal{D}} \Big[ &\underbrace{\mathbb{E}_{z \sim p_\vartheta(\cdot|s,a)} [-\log q_\psi(s' | z)]}_{\text{Reconstruction}} \nonumber \\
    &+ \beta \Big( \underbrace{D_{\mathrm{KL}}(p_\vartheta(z|s,a) \| p_\varphi(z|s))}_{\text{Bidirectional Action IB}} + \underbrace{D_{\mathrm{KL}}(p_\varphi(z|s) \| p(z))}_{\text{State IB}} \Big) \Big]. 
\end{align}

\begin{proof}
To learn a predictive yet sufficient representation, we adopt the VIB framework based on the causal diagram $(s, a) \to z \to s'$. The core objective is to maximize the mutual information between the latent representation $Z$ and the target next state $S'$, while constraining the information $Z$ retains about the state-action pair $(S, A)$:
\begin{align}
    \max I(Z; S') - \beta I(Z; S, A)
\end{align}
We will examine each of these terms in turn.

\textbf{Predictive Lower Bound for $I(Z; S')$.}
We first derive a tractable lower bound for the predictive term $I(Z; S')$. Expanding the mutual information, we have:
\begin{align}
    I(Z; S') &= \mathbb{E}_{s'\sim p(s', z)} \Big[ \log \frac{p(s'|z)}{p(s')} \Big] \\
    &= \mathbb{E}_{p(s', z)} [ \log p(s'|z) ] + H(S')
\end{align}
Since the true posterior $p(s'|z)$ is intractable, we introduce a parameterized variational decoder $q_\psi(s'|z)$. Using the non-negativity of the KL divergence, $D_{\mathrm{KL}}(p(s'|z) \| q_\psi(s'|z)) \ge 0$, we obtain $\mathbb{E}_{p(s'|z)}[\log p(s'|z)] \ge \mathbb{E}_{p(s'|z)}[\log q_\psi(s'|z)]$. Substituting this yields:
\begin{align}
    I(Z; S') \ge \mathbb{E}_{p(s', z)} [\log q_\psi(s'|z)] + H(S')
\end{align}
The entropy term $H(S')$ depends only on the environment dynamics and is constant with respect to our optimization parameters. Thus, it can be ignored during training. We can rewrite the joint distribution $p(s', z)$ by marginalizing over the dataset $\mathcal{D}$. Leveraging the causal diagram that $S'$ and $Z$ are conditionally independent given $(S, A)$, we have $p(s, a, s', z) = p(s, a) p(s'|s, a) p_\vartheta(z|s, a)$. This yields the standard lower bound of the mutual information between $Z$ and $S'$:
\begin{equation}
    I(Z; S') \ge \mathbb{E}_{(s,a,s') \sim \mathcal{D}} \left[ \mathbb{E}_{z \sim p_\vartheta(\cdot|s,a)} [\log q_\psi(s'|z)] \right] + \text{const.}
\end{equation}

\textbf{Complexity Upper Bound for $I(Z; S, A)$.}
Next, we establish a tractable upper bound for the complexity penalty $I(Z; S, A)$. To introduce a state belief, we can decompose this term using the chain rule of mutual information:
\begin{equation}
    I(Z; S, A) = I(Z; S) + I(Z; A | S)
\end{equation}
Let $p_{\pi_{\beta}}(z|s) = \int p_\vartheta(z|s,a)\pi_{\beta}(a|s)da$ be the true state belief (state-conditional marginal representation) under the behavior policy $\pi_\beta$, and $p(z) = \int p_{\pi_{\beta}}(z|s)p(s)ds$ be the fully marginalized prior. We can expand the mutual information as the expected KL divergence:
\begin{align}
    I(Z; S, A) &= \mathbb{E}_{(s,a)\sim\mathcal{D}} \left[ D_{\mathrm{KL}}(p_\vartheta(z|s,a) \| p(z)) \right] \nonumber \\
    &= \mathbb{E}_{(s,a)\sim\mathcal{D}} \left[ \int p_\vartheta(z|s,a) \log \left( \frac{p_\vartheta(z|s,a)}{p_{\pi_{\beta}}(z|s)} \cdot \frac{p_{\pi_{\beta}}(z|s)}{p(z)} \right) dz \right] \nonumber \\
    &= \underbrace{\mathbb{E}_{(s,a)\sim\mathcal{D}} \left[ D_{\mathrm{KL}}(p_\vartheta(z|s,a) \| p_{\pi_{\beta}}(z|s)) \right]}_{\text{(1) } I(Z; A|S)} + \underbrace{\mathbb{E}_{s\sim\mathcal{D}} \left[ D_{\mathrm{KL}}(p_{\pi_{\beta}}(z|s) \| p(z)) \right]}_{\text{(2) } I(Z; S)}
\end{align}

Because both $p_{\pi_{\beta}}(z|s)$ and $p(z)$ are computationally intractable, we introduce variational approximations: a learned state belief network $p_\varphi(z|s)$ to approximate $p_{\pi_{\beta}}(z|s)$, and a fixed standard prior $p(z)$ (typically $\mathcal{N}(0, I)$).

To bound term (1) $I(Z; A|S)$, we introduce the state belief network $p_\varphi(z|s)$ and rewrite the term using the non-negativity property of the KL divergence $D_{\mathrm{KL}}(p_{\pi_{\beta}} \| p_\varphi) \ge 0$. This gives us a rigorous upper bound:
\begin{align}
    I(Z; A|S) &= \mathbb{E}_{(s,a)\sim\mathcal{D}} \left[ \int p_\vartheta(z|s,a) \log \frac{p_\vartheta(z|s,a)}{p_\varphi(z|s)} dz - \int p_\vartheta(z|s,a) \log \frac{p_{\pi_{\beta}}(z|s)}{p_\varphi(z|s)} dz \right] \\
    &= \mathbb{E}_{(s,a)\sim\mathcal{D}} \left[ D_{\mathrm{KL}}(p_\vartheta(z|s,a) \| p_\varphi(z|s)) \right] - \mathbb{E}_{s\sim\mathcal{D}} \left[ D_{\mathrm{KL}}(p_{\pi_{\beta}}(z|s) \| p_\varphi(z|s)) \right] \\
    &\le \mathbb{E}_{(s,a)\sim\mathcal{D}} \left[ D_{\mathrm{KL}}(p_\vartheta(z|s,a) \| p_\varphi(z|s)) \right] \label{eq:state_belief}
\end{align}

To address term (2) $I(Z; S)$, replacing the true marginal $p(z)$ with a prior $\mathcal{N}(0,I)$ yields a valid upper bound $\mathbb{E}_{s\sim\mathcal{D}} \left[ D_{\mathrm{KL}}(p_{\pi_{\beta}}(z|s) \| p(z)) \right]$, since the true marginal minimizes this KL divergence:
\begin{align}
I(Z; S) &\le \mathbb{E}_{s\sim\mathcal{D}} \left[ D_{\mathrm{KL}}(p_{\pi_{\beta}}(z|s) \| p(z)) \right].
\end{align}

However, the $p_{\pi_{\beta}}(z|s)$ remains intractable. Fortunately, as indicated by Eq. \eqref{eq:state_belief}, optimizing the upper bound for $I(Z; A|S)$ implicitly drives $p_\varphi(z|s)$ to be a tightly optimized variational approximation of $p_{\pi_{\beta}}(z|s)$: $D_{\mathrm{KL}}(p_{\pi_{\beta}}(z|s) \| p_\varphi(z|s))\le \mathbb{E}_{a \sim \pi_\beta(a|s)} \left[ D_{\mathrm{KL}}( p_\vartheta(z|s, a) \| p_{\varphi}(z|s) ) \right] 
$ (due to the convexity of the KL divergence). 
Substituting the intractable marginal $p_{\pi_\beta}(z|s)$ with the learned surrogate $p_\varphi(z|s)$ breaks the strict analytical inequality, meaning the resulting formula acts as a heuristic approximation rather than a rigorous bound. However, it remains a principled and widely adopted variational technique in practice to achieve tractability. Consequently, we formulate the following surrogate objective for term (2): $\mathcal{L}_{I(Z; S)} = \mathbb{E}_{s\sim\mathcal{D}} \left[ D_{\mathrm{KL}}(p_\varphi(z|s) \| p(z)) \right].$


\textbf{Final Objective.}
Combining the rigorous lower bound for $I(Z; S')$, the rigorous upper bound for $I(Z; A|S)$, and the tractable variational approximation for $I(Z; S)$, the original IB objective is approximated by our final empirical objective $\mathcal{L}_{\mathrm{SBGF}}$:
\begin{align}
    \mathcal{L}_{\mathrm{SBGF}}(\psi, \vartheta, \varphi) = \mathbb{E}_{(s,a,s') \sim \mathcal{D}} \Big[ &\mathbb{E}_{z \sim p_\vartheta(\cdot|s,a)} [-\log q_\psi(s' | z)] \nonumber \\
    &+ \beta \big( D_{\mathrm{KL}}(p_\vartheta(z|s,a) \| p_\varphi(z|s)) + D_{\mathrm{KL}}(p_\varphi(z|s) \| p(z)) \big) \Big],
\end{align}
where the latent variable $z$ can be sampled using the reparameterization trick \citep{kingma2013auto}. We use an encoder of the form $p_\vartheta(z|s,a) = \mathcal{N}(\mu_\vartheta, \sigma_\vartheta|s,a)$ and state belief $p_\varphi(z|s) = \mathcal{N}(\mu_\varphi, \sigma_\varphi|s)$; the gradients can be easily backpropagated through the network parameters $\psi, \vartheta,$ and $\varphi$. Specifically, the two KL divergences have the following closed-form solutions:
\begin{align}
    D_{\mathrm{KL}}(p_\vartheta(z|s, a) \| p_\varphi(z|s)) &= -\frac{1}{2}(1 + \log\frac{\sigma_{\vartheta}^2}{\sigma_{\varphi}^2}-\frac{\sigma_{\vartheta}^2+(\mu_{\vartheta} - \mu_{\varphi})^2}{\sigma_{\varphi}^2})\\
    D_{\mathrm{KL}}(p_\varphi(z|s) \| p(z)) &= -\frac{1}{2}(1+\log\sigma_{\varphi}^2 - \mu_{\varphi}^2 - \sigma_{\varphi}^2)
\end{align}
This concludes the proof.
\end{proof}

\subsection{Proof of Theorem 1}\label{sec:app_t1}

\textbf{Theorem 1} \textit{(Surrogate Discrepancy Bound).} \textit{Let $\hat p_{\pi}(z|s)=\mathbb{E}_{a\sim\pi(\cdot|s)} p_\vartheta(z|s,a)$ for any given policy $\pi$, and let $\hat p_{\pi_\beta}(z|s)= p_{\varphi}(z|s)$ be the models estimated by SBGF in Eq.~\eqref{eq:sbvib} on an offline dataset $\mathcal{D}$ of size $N$. With probability at least $1 - \delta$, the following bound holds for the discrepancy between the true consequence divergence and the empirical surrogate divergence:}
\begin{small}
\begin{align}
    &\left| \mathbb{E}_{s \sim d^{\pi_\beta}} \big[ D_{\mathrm{KL}}(P_{\pi}(s'|s) \| P_{\pi_\beta}(s'|s)) \big] - \mathbb{E}_{s \sim \mathcal{D}} \big[ D_{\mathrm{KL}}(\hat p_{\pi}(z|s)  \| \hat p_{\pi_\beta}(z|s)) \big] \right| \\
    \le &\underbrace{\mathbb{E}_{s,s'} \left[ D_{\mathrm{KL}}(p_{\pi}(z|s,s') \| p_{\pi_\beta}(z|s,s')) \right]}_{\substack{\text{\textbf{(I) Dynamics Info Compression (Bias)}} \\ \text{\textbf{(Increases as }} \beta ~\uparrow \text{\textbf{)}}}}  
    + \underbrace{\mathcal{O}\left( \sqrt{\frac{\mathrm{Var}_s(D_{\mathrm{KL}}(p_{\pi}(z|s)  \| p_{\pi_\beta}(z|s)) \log(1/\delta)}{N}} \right)}_{\substack{\text{\textbf{(II) Estimation Error (Variance)}} \\ \text{\textbf{(Decreases as }} \beta ~\uparrow \text{\textbf{)}}}} 
    + \underbrace{C_\pi \epsilon_{\mathrm{opt},\beta}}_{\text{\textbf{(III) Policy Shift}}} 
\end{align}
\end{small}
where the policy ratio in the latent space is bounded by $C_\pi = \max_{s \in \mathcal{D}, z} \frac{p_\pi(z|s)}{p_{\pi_\beta}(z|s)}$, and $\epsilon_{\mathrm{opt},\beta}$ is the optimization error related to the bidirectional action IB loss in Eq.~\eqref{eq:sbvib}. 

\begin{proof}

The goal of our analysis is to bound the discrepancy between our learned empirical surrogate latent divergence and the true consequence divergence.

We begin by simplifying the notation for the divergence metrics involved in our analysis:
\begin{itemize}
    \item True Consequence Divergence $D_{s'} := \mathbb{E}_{s \sim d^{\pi_\beta}} \big[ D_{\mathrm{KL}}(P_\pi(s'|s) \| P_{\pi_\beta}(s'|s)) \big]$, representing the expected divergence between consequence distributions across all states.
    \item Theoretical Latent Divergence $D_{z} := \mathbb{E}_{s \sim d^{\pi_\beta}} \big[ D_{\mathrm{KL}}(p_\pi(z|s) \| p_{\pi_\beta}(z|s)) \big]$. Here, the two distributions represent the theoretical consequence beliefs based on the encoder for the learned policy and the behavior policy, respectively: $p_{\pi_\beta}(z|s) := \mathbb{E}_{a \sim \pi_\beta(\cdot|s)} [p_\vartheta(z|s,a)]$ and $p_{\pi}(z|s) := \mathbb{E}_{a \sim \pi(\cdot|s)} [p_\vartheta(z|s,a)]$. 
    \item Empirical Latent Divergence $\hat{D}_{z} := \mathbb{E}_{s \sim \hat{d}^{\pi_\beta}} \big[ D_{\mathrm{KL}}(\hat{p}_\pi(z|s) \| \hat{p}_{\pi_\beta}(z|s)) \big]$, where $\hat{d}^{\pi_\beta}(s)$ is the empirical state distribution defined by the offline dataset $\mathcal{D}$ of size $N$. The two empirical distributions are based on our encoder and the learned state belief: $\hat{p}_\pi(z|s) := \int p_\vartheta(z|s,a)\pi(a|s)\,da$ and $\hat{p}_{\pi_\beta}(z|s) := p_\varphi(z|s)$.
\end{itemize}

Our first goal is to bound the discrepancy between the intractable objective $D_{s'}$ and the objective optimized in practice $\hat{D}_{z}$. By introducing the theoretical latent divergence $D_{z}$ and applying the triangle inequality, we decompose the total error into a theoretical gap and an empirical error:
\begin{align} 
    |D_{s'} - \hat{D}_{z}| 
    &= |D_{s'} - D_{z} + D_{z} - \hat{D}_{z}| \\ 
    &\le \underbrace{|D_{s'} - D_{z}|}_{\text{Theoretical Gap}} + \underbrace{|D_{z} - \hat{D}_{z}|}_{\text{Empirical Error}}.
\end{align}

We proceed to analyze these terms individually.

\textbf{For the Theoretical Gap.} 
Taking the expectation over $s \sim d^{\pi_\beta}(s)$ and applying the exact gap derived in Eq.~\eqref{eq:exact_gap}, the theoretical gap is exactly the expected posterior divergence:
\begin{align}
    |D_{s'} - D_{z}| = \mathbb{E}_{s \sim d^{\pi_\beta}(s), s' \sim P_\pi(\cdot|s)} \left[ D_{\mathrm{KL}}(p_\pi(z|s, s') \| p_{\pi_\beta}(z|s, s')) \right],
\end{align}
where $p_\pi(z | s, s') = \frac{p(s'|z)p_\pi(z|s)}{P_\pi(s'|s)}$ represents the posterior of the latent variable $z$ conditioned on the next state $s'$. 

\textbf{For the Empirical Error.} 
To further simplify notation, let $D(s) := D_{\mathrm{KL}}(p_\pi(z|s) \| p_{\pi_\beta}(z|s))$ and $\hat{D}(s) := D_{\mathrm{KL}}(\hat{p}_\pi(z|s) \| \hat{p}_{\pi_\beta}(z|s))$. By applying the triangle inequality again, we separate the empirical error into a sampling error due to finite data, and an estimation error due to model approximation:
\begin{align}
    |D_{z} - \hat{D}_{z}| 
    &= \left| \int ds \, d^{\pi_\beta}(s) D(s) - \int ds \, \hat{d}^{\pi_\beta}(s) \hat{D}(s) \right| \\
    &\le \underbrace{\left| \int ds \left(d^{\pi_\beta}(s) - \hat{d}^{\pi_\beta}(s) \right) D(s) \right|}_{\text{Sampling Error}} + \underbrace{\int ds \, \hat{d}^{\pi_\beta}(s) \left| D(s) - \hat{D}(s) \right|}_{\text{Estimation Error}}.
\end{align}

\textbf{Bounding Sampling Error.} 
In the continuous-control setting, we consider $D_z = \mathbb{E}_{s \sim d^{\pi_\beta}(s)} [D(s)]$ as a bounded expectation under mild encoder regularity. Assuming there exists a constant $D_{max}$ such that $|D(s) - D_z| \le D_{max}$ similar to~\cite{ib-bound}, we can apply the Bernstein inequality (following Lemma A.3 in~\cite{bcp}). With probability at least $1 - \delta$, the sampling error is bounded by:
\begin{align}
\text{Sampling Error} \le \sqrt{\frac{2 \mathrm{Var}_{s \sim d^{\pi_\beta}}(D(s)) \log(2/\delta)}{N}} + \frac{2D_{\max} \log(2/\delta)}{3N}
\end{align}

For sufficiently large $N$, the second term decays faster, making the first term dominant. Using big-$\mathcal{O}$ notation, the bound simplifies to: 
\begin{align}
\text{Sampling Error} \le \mathcal{O}\left( \sqrt{\frac{\mathrm{Var}_s(D(s)) \log(1/\delta)}{N}} \right)
\end{align}


\textbf{Bounding Estimation Error.}
The estimation error relies on the capacity of the learned models. By definition above, we decompose the discrepancy as follows:
\begin{align}
    \text{Estimation Error} &= \mathbb{E}_{s \sim \mathcal{D}}\left| D(s) - \hat{D}(s) \right|\\
    &= \mathbb{E}_{s \sim \mathcal{D}} \left| D_{\mathrm{KL}}(p_\pi(z|s) \| p_{\pi_\beta}(z|s)) - D_{\mathrm{KL}}(\hat{p}_\pi(z|s) \| \hat{p}_{\pi_\beta}(z|s)) \right| \nonumber \\
    &\overset{(a)}{=} \mathbb{E}_{s \sim \mathcal{D}} \left| \int p_\pi(z|s) \log \frac{\hat{p}_{\pi_\beta}(z|s)}{p_{\pi_\beta}(z|s)} dz \right| \\
    &\le \mathbb{E}_{s \sim \mathcal{D}} \left[ \int p_\pi(z|s) \left| \log \frac{p_{\pi_\beta}(z|s)}{\hat{p}_{\pi_\beta}(z|s)} \right| dz \right] \quad \text{(Triangle Inequality)}\\
    &\overset{(b)}{\le} C_\pi \mathbb{E}_{s \sim \mathcal{D}} \left[ \int p_{\pi_\beta}(z|s) \left| \log \frac{p_{\pi_\beta}(z|s)}{\hat{p}_{\pi_\beta}(z|s)} \right| dz \right] \\
    &\overset{(c)}{=} C_\pi \mathbb{E}_{s \sim \mathcal{D}} \left[ \int p_{\pi_\beta}(z|s) \log \frac{p_{\pi_\beta}(z|s)}{\hat{p}_{\pi_\beta}(z|s)} dz - 2 \int_{p_{\pi_\beta} < \hat{p}_{\pi_\beta}} p_{\pi_\beta}(z|s) \log \frac{p_{\pi_\beta}(z|s)}{\hat{p}_{\pi_\beta}(z|s)} dz \right] \\
    &\overset{(d)}{\le} C_\pi \mathbb{E}_{s \sim \mathcal{D}} \left[ D_{\mathrm{KL}}(p_{\pi_\beta}(z|s) \| \hat{p}_{\pi_\beta}(z|s)) + \int \left| p_{\pi_\beta}(z|s) - \hat{p}_{\pi_\beta}(z|s) \right| dz \right] \\
    &\le C_\pi \mathbb{E}_{s \sim \mathcal{D}} \left[ D_{\mathrm{KL}}(p_{\pi_\beta} \| \hat{p}_{\pi_\beta}) + \sqrt{2 D_{\mathrm{KL}}(p_{\pi_\beta} \| \hat{p}_{\pi_\beta})} \right]. \quad \text{(Pinsker's inequality)}
\end{align}

For (a): Based on an encoder defined by a given $\beta$, the policy-conditioned latent distribution is strictly formed by integrating over actions: $\hat{p}_\pi(z|s) = \int p_\vartheta(z|s,a)\pi(a|s)da = p_\pi(z|s)$. Expanding the KL divergences leaves the $\log\frac{\hat{p}_{\pi_\beta}}{ p_{\pi_\beta}}$.

For (b): We apply the triangle inequality to move the absolute value inside the integral. We apply importance sampling $p_\pi(z|s) = \frac{p_\pi(z|s)}{p_{\pi_\beta}(z|s)} p_{\pi_\beta}(z|s)$, and bound the density ratio in latent space with a constant $C_\pi = \max_{s \in \mathcal{D}, z} \frac{p_\pi(z|s)}{p_{\pi_\beta}(z|s)}$, representing the maximum policy shift in latent space.

For (c): We rewrite the absolute log term using the identity $|x| = x - 2x \cdot \mathbb{I}_{\{x < 0\}}$. 

For (d): The first integral is exactly the definition of KL divergence. For the second integral, when $p_{\pi_\beta} < \hat{p}_{\pi_\beta}$, the negative log inequality $-\log(p/q) \le (q/p) - 1$ implies that $-2 \int_{p<q} p \log(p/q) \le 2 \int_{p<q} (q - p)$. Using a standard property of probability distributions $2 \int_{p<q} (q - p) = \int |p - q| dz$ yields the absolute difference integral term.

Recalling that $\hat{p}_{\pi_\beta}(z\vert{}s) := p_\varphi(z\vert{}s)$, we introduce the optimal projection $ p_{\varphi}^\star(z|s)$, defined as $\varphi^\star= \arg\min_\varphi \mathbb{E}_{s}[D_{KL}(p_{\pi_\beta} | p_\varphi)]$, we can decompose the error into two distinct components:
\begin{align}
&\mathbb{E}_{s \sim \mathcal{D}} [D_{KL}(p_{\pi_\beta}(z|s) \| p_\varphi(z|s))] \\
= &\underbrace{\mathbb{E}_{s \sim \mathcal{D}} [D_{KL}(p_{\pi_\beta}(z|s)\|p_{\varphi^\star}(z|s))]}_{\text{realizability error }(\epsilon_\beta)} + \underbrace{\mathbb{E}_{s \sim \mathcal{D}} [\mathbb{E}_{z \sim p_{\pi_\beta}(\cdot|s)} [\log p_{\varphi^\star}(z|s) - \log p_\varphi(z|s)]]}_{\text{optimization error }(\epsilon_{\mathrm{opt}})}
\end{align}
The first term explicitly captures the realizability gap ($\epsilon_{\beta}$). This represents the fundamental approximation limit of fitting a single Gaussian to a potentially multimodal behavior-induced latent distribution. This term depends on $\beta$ and vanishes only when $\beta$ is sufficiently large, causing $p_\vartheta$ to collapse to a single Gaussian distribution across all actions. The second term now cleanly isolates the true optimization error ($\epsilon_{\mathrm{opt}}$), which stems strictly from the model training process.

We use Jensen's Inequality on concave function $\sqrt{\cdot}$ and simplify notation of the estimation error bound as:
\begin{align}
\text{Estimation Error} \le C_\pi \left((\epsilon_\beta+\epsilon_{\mathrm{opt}}) + \sqrt{2( \epsilon_\beta+\epsilon_{\mathrm{opt}})}\right):=C_\pi \epsilon_{\mathrm{opt},\beta}
\end{align}

Combining all components yields the final bound:

\begin{align}
    &\left|\mathbb{E}_{s \sim d^{\pi_\beta}} \big[ D_{\mathrm{KL}}(P_\pi(s'|s) \| P_{\pi_\beta}(s'|s)) \big] - \mathbb{E}_{s \sim \hat{d}^{\pi_\beta}(s)} \big[D_{\mathrm{KL}}(\mathbb{E}_{a\sim\pi(\cdot|s)} p_\vartheta(z|s, a) \| p_\varphi(z|s))\big]\right| \\
    \le &\mathbb{E}_{s,s'} \left[ D_{\mathrm{KL}}(p_\pi(z|s,s') \| p_{\pi_\beta}(z|s,s')) \right]
    + \mathcal{O}\left( \sqrt{\frac{\mathrm{Var}_s(D_{\mathrm{KL}}(p_{\pi}(z|s)  \| p_{\pi_\beta}(z|s)) \log(1/\delta)}{N}} \right) \nonumber \\ 
    &+ C_\pi \epsilon_{\mathrm{opt},\beta} 
\end{align}

\end{proof}


\textbf{Remark} \textit{(Validating the Bias-Variance Trade-off ($\beta$)).} 
Theorem~\ref{thm:surrogate_bound} reveals that our surrogate approximation is systematically governed by a fundamental Bias-Variance trade-off dictated by the information compression parameter $\beta$, guiding the representation learning in the causal diagram $s, a \to z \to s'$.

\begin{itemize}
    \item \textbf{Term (I) acts as the Bias.} It measures the hidden divergence of the \textit{policy latent consequences} conditioned on the \textit{same physical outcome} $s'$. When $\beta$ is sufficiently small, the model primarily maximizes the predictive information $I(Z; S')$, making the transition mapping $z \to s'$ near-deterministic. Consequently, given the start state $s$ and the end state $s'$, the intermediate latent $z$ is almost uniquely determined. In this regime, the posterior $p_\pi(z|s,s')$ becomes largely independent of the specific policy, keeping Term (I) close to zero. However, as $\beta$ moderately increases, the information bottleneck introduces stochasticity and compresses out action-specific details. Multiple different latent representations $z$ (arising from different actions) might now map to the same physical outcome $s'$. Consequently, given $(s, s')$, the inferred $z$ relies more heavily on the underlying policy's action distribution, which inflates the divergence between $p_\pi(z|s,s')$ and $p_{\pi_\beta}(z|s,s')$. \textit{(While an excessively large $\beta$ would force $z$ to collapse entirely to the prior, trivially erasing policy differences and bringing Term (I) back to zero, such a heavily regularized model loses all dynamics prediction capability. Thus, the meaningful bias-variance trade-off occurs in the regime of relatively small $\beta$.)}
    \item \textbf{Term (II) acts as the Variance.} It captures the fluctuation of the latent policy discrepancy across different states. A large $\beta$ acts as a strong regularizer, heavily penalizing the action information $a$ in $z$. This smooths out the representational differences between $\pi$ and $\pi_\beta$, effectively suppressing the cross-state variance. A small $\beta$ preserves action details, increasing this variance.
    \item \textbf{Term (III) highlights the OOD Penalty.} $C_\pi$ represents the coverage coefficient of the evaluated policy $\pi$ relative to the behavior policy $\pi_\beta$ in the latent space. If an OOD action maps far away from the support of $p_{\pi_\beta}(z|s)$, $C_\pi$ can become unbounded. In such cases, Theorem 1's bound fails, and the surrogate metric becomes ineffective at guaranteeing the consequence distance. This establishes the theoretical necessity of our proposed regularizer. By explicitly constraining the learned policy $\pi$ to stay within the latent support of $\pi_\beta$, the regularizer ensures that $C_\pi$ remains bounded. As long as OOD behavior in the latent space is prevented, the surrogate's approximation of the consequence divergence is theoretically guaranteed.
\end{itemize}

\textit{Guideline for tuning $\beta$:} This theoretical trade-off provides a principled guideline for environmental adaptation. In highly stochastic environments, the inherent variance of Term (II) dominates, necessitating a larger $\beta$ to curb Term (II). In near-deterministic environments, a smaller $\beta$ is optimal to strictly preserve the transition dynamics and minimize the bias of Term (I).

\subsection{Proof Sketch of Theorem 1: Implementation for Policy Optimization} \label{sec:app_t1s}
In practical policy optimization, directly minimizing the empirical latent divergence $D_{\mathrm{KL}}(\hat{p}_\pi(z|s) \| \hat{p}_{\pi_\beta}(z|s)) = D_{\mathrm{KL}}(\mathbb{E}_{a\sim\pi(\cdot|s)} [p_\vartheta(z|s,a)] \| {p}_{\varphi}(z|s))$ is computationally intractable. The optimization involves nested Monte Carlo sampling, which is expensive and prone to biased estimation. To address this, we instead optimize a tractable point-wise upper bound. By expanding the divergence, we have:
\begin{align}
&~~~~~\mathbb{E}_{a\sim\pi(\cdot|s)} \big[D_{\mathrm{KL}}(p_\vartheta(z|s, a) \| p_\varphi(z|s))\big] \\
&= \int_a \pi(a|s) \int_z p_\vartheta(z|s,a) \log \frac{p_\vartheta(z|s,a)}{p_\varphi(z|s)} \, dz \, da \\
&= \int_a \int_z \pi(a|s) p_\vartheta(z|s,a) \log \left( \frac{p_\vartheta(z|s,a)}{\hat{p}_\pi(z|s)} \cdot \frac{\hat{p}_\pi(z|s)}{p_\varphi(z|s)} \right) dz \, da \\
&= \int_a \int_z \pi(a|s) p_\vartheta(z|s,a) \log \frac{p_\vartheta(z|s,a)}{\hat{p}_\pi(z|s)} dz \, da+ \int_a \int_z \pi(a|s) p_\vartheta(z|s,a) \log \frac{\hat{p}_\pi(z|s)}{p_\varphi(z|s)} dz \, da\\
&= I_\pi(A;Z|s) + D_{\mathrm{KL}}\big(\hat{p}_\pi(z|s) \| p_\varphi(z|s)\big)\label{eq:I(az)}
\end{align}

where $I_\pi(A;Z|s)$ denotes the mutual information between the action $a$ and the latent consequence $z$ induced by the policy $\pi$ conditioned on state $s$. Since mutual information is strictly non-negative, the left-hand side serves as a rigorous upper bound to the empirical latent divergence. By minimizing this upper bound, we can fuse the latent constraint with the Monte Carlo sampling of the Q-value, resulting in our final, highly efficient optimization objective in Eq.~\eqref{eq:pi_practical}.

Crucially, utilizing this upper bound is not merely a computational compromise, but a highly desirable mechanism. The gap between the original marginal KL and our surrogate is exactly the mutual information penalty $I_\pi(A;Z|S)$:
\begin{align}
|D_{s'} - \hat{D}_{po}| 
&\le   |D_{s'} - \hat{D}_{z}| + |\hat{D}_{z}- \hat{D}_{po}|  \quad \text{(Triangle Inequality)} \\
& = |D_{s'} - \hat{D}_{z}| + I_\pi(A;Z|S) \quad\ \text{(Eq.~\eqref{eq:I(az)})} 
\end{align}
where $\hat{D}_{po} := \mathbb{E}_{s \sim \hat{d}^{\pi_\beta}(s),a\sim\pi(\cdot|s)} \big[D_{\mathrm{KL}}(p_\vartheta(z|s, a) \| p_\varphi(z|s))\big]$.
This yields the bound:
\begin{align}
    &\left|\mathbb{E}_{s \sim d^{\pi_\beta}} \big[ D_{\mathrm{KL}}(P_\pi(s'|s) \| P_{\pi_\beta}(s'|s)) \big] - \mathbb{E}_{s \sim \hat{d}^{\pi_\beta}(s),a\sim\pi(\cdot|s)} \big[D_{\mathrm{KL}}(p_\vartheta(z|s, a) \| p_\varphi(z|s))\big] \right| \\
    \le &\mathbb{E}_{s,s'} \left[ D_{\mathrm{KL}}(p_\pi(z|s,s') \| p_{\pi_\beta}(z|s,s')) \right]
    + \mathcal{O}\left( \sqrt{\frac{\mathrm{Var}_s(D_{\mathrm{KL}}(p_{\pi}(z|s)  \| p_{\pi_\beta}(z|s)) \log(1/\delta)}{N}} \right) + C_\pi \epsilon_{\mathrm{opt},\beta} + I_\pi(A;Z|S)
\end{align}

Minimizing the mutual information penalty $I_\pi(A;Z|S)$ implicitly enforces action–consequence consistency, thereby resolving a critical risk-dilution issue. Specifically, solely matching marginal distributions allows dangerous OOD actions to be masked by expectation averaging, rendering the aggregate distribution $\hat{p}_\pi(z|s)$ deceptively safe. In contrast, penalizing this term enforces a point-wise constraint: every individually sampled action must strictly align with the safe behavior distribution, fundamentally preventing catastrophic deviations.

However, tightening this point-wise KL upper bound introduces a structural trade-off:
On the one hand, it significantly benefits tasks where the optimal policy is near-deterministic. As illustrated in Figure~\ref{fig:exp3}, stochastic policies often assign non-negligible probability mass to sub-optimal actions, leading to inferior performance compared to deterministic policies. On the other hand, excessively penalizing policy-conditioned divergence risks entropy collapse by overly penalizing stochasticity, which may induce premature convergence and severe overfitting. Consequently, maintaining necessary policy stochasticity requires incorporating explicit entropy regularization.

\newpage

\section{Implementation Details and Extended Experiments}\label{sec:app_experiments}

\subsection{Benchmark Experiments}\label{sec:app_main}

\paragraph{Data.} 
We use the datasets from the D4RL benchmark~\citep{d4rl} of the latest versions, which are v2 for MuJoCo, AntMaze and v1 for Adroit. We provide an explanation of the abbreviations used in the Figure~\ref{fig:exp3}. Mujoco Tasks: Hc = halfcheetah, H = hopper, W = walker2d, m = medium, mr = medium-replay, me = medium-expert. AntMaze Tasks: u = AntMaze-umaze, u-d = AntMaze-umaze-diverse, m-p = AntMaze-medium-play, m-d = AntMaze-medium-diverse, l-p = AntMaze-large-play, l-d = AntMaze-large-diverse. Adroit Tasks: h = human, hamm = hammer.

\paragraph{Baselines.} 
We report the baseline results from their original papers, and reproduce the missing experiments using the official code:
\begin{itemize}
    \item TD3BC~\citep{td3bc}: \url{https://github.com/sfujim/TD3_BC}
    \item SPOT~\citep{spot}: \url{https://github.com/thuml/SPOT}
    \item FQL~\citep{fql}: \url{https://github.com/seohongpark/fql}
\end{itemize}

\paragraph{Implementation details.} 
We normalize the states in the dataset for MuJoCo and Adroit but do not normalize the states for AntMaze. Following CQL~\citep{cql}, we subtract 1 from rewards for the AntMaze datasets. Hyperparameters used are presented in Table \ref{tab:hyper}. For policy evaluation, we average mean returns over 10 evaluation trajectories on the MuJoCo and Adroit tasks, and average over 100 evaluation trajectories on the AntMaze tasks. For stochastic policies implementation in Section~\ref{sec:ablation}, we directly replace the flow matching with our latent divergence. Our code is available at: \url{https://github.com/DarrenJum/StBR}. 

To prevent the Q-value gradients from dominating the KL regularization term, we adopt the normalization strategy like most policy constraint methods~\citep{td3bc,spot,cped}. We scale the Q-values by an absolute average over the current minibatch:
$ \mathcal{L}_\pi(\phi) = \mathbb{E}_{s \sim \mathcal{D},a \sim \pi_\phi(\cdot|s)} \left[ -\frac{Q_\theta(s, a)}{\alpha} + \lambda D_{\mathrm{KL}}(p_\vartheta(z|s, a) \| p_\varphi(z|s)) \right], $
where the scaling factor $\alpha = \frac{1}{B} \sum_{i=1}^B |Q_\theta(s_i, a_i)|$ is dynamically computed from a minibatch of $B$ transitions, ensuring stable optimization across different tasks.

\paragraph{Explicit Next-State Prediction (ENSP) Baseline.}
To demonstrate the advantage of constraining policy updates in the latent space rather than the raw observation space, we construct an explicit next-state prediction baseline. Instead of matching probability distributions via KL divergence, this baseline directly penalizes deviations in the state space using an explicit deterministic dynamics model $\hat{s}' = \hat{P}(s, a)$, pre-trained on transitions $(s, a, s') \sim \mathcal{D}$ via mean squared error. During policy optimization, we regularize the policy by pulling the predicted next state under the policy action toward the empirical transition target $s'$ recorded in the dataset:
\begin{equation}
    \max_{\pi_\phi} \mathbb{E}_{(s, s') \sim \mathcal{D}, a \sim \pi_\phi(\cdot|s)} \left[ \frac{Q_\theta(s, a)}{\alpha} - \lambda_{ENSP} \left\| \hat{P}(s, a) - s' \right\|_2^2 \right],
\end{equation}
where $\alpha$ is the batch-adaptive scaling factor defined as above. 

\begin{algorithm}[h]
    \caption{Structural Behavior Regularization (StBR)}
    \label{alg:sbr}
    \textbf{Input}: The offline dataset $\mathcal{D} = \{(s, a, r, s')\}$
    \begin{algorithmic}[1]
        \STATE Initialize SBGF with parameters $\psi$, $\vartheta$ and $\varphi$, policy network $\pi_\phi$, critic network $Q_\theta$ and target network $Q_{\bar{\theta}}$ with $\bar{\theta} \leftarrow \theta$ \\
        \FOR{$t = 1$ to $T$}
            \STATE Sample mini-batch of transitions $(s,a,r,s') \sim \mathcal{D}$
            \STATE Update $\theta$ minimizing $\mathcal{L}_Q(\theta)$ in Eq.~\eqref{eq:empirical_T}
            \STATE Update $\phi$ minimizing $ \mathcal{L}_\pi(\phi)$ in Eq.~\eqref{eq:pi_practical}
            \STATE Update $\psi, \vartheta, \varphi$ minimizing $\mathcal{L}_{\mathrm{SBGF}}(\psi, \vartheta, \varphi)$ in Eq.~\eqref{eq:sbvib}
            \STATE Update target network: $\bar{\theta}\leftarrow\tau\theta+(1-\tau)\bar{\theta}$
        \ENDFOR
    \end{algorithmic}
\end{algorithm}

\begin{table}[h]
\centering
\caption{Hyperparameters in StBR}
\begin{tabular}{cll}
\toprule
\multicolumn{1}{l}{}           & Hyperparameter                & Value                 \\ \hline
\multirow{13}{*}{StBR}          & Optimizer                     & Adam                   \\
                               & Activation function& ReLU                   \\
                               & Critic learning rate          & $3\times10^{-4}$         \\
                               & Actor learning rate           & $3\times10^{-4}$        \\
                               & Batch size                    & 256                    \\
                               & Discount factor               & 0.99                      \\
                               & Gradient Steps                & $1\times10^{6}$          \\
                               & SBGF learning rate            & $1\times10^{-3}$         \\
                               & Target network update rate    & 0.005                      \\
                               & Policy update frequency       & 2                          \\
                               & Number of Critics             & 4 (including 2 target networks)     \\
                               & $\beta$             & 0.001 for MuJoCo, 0.01 for AntMaze and Adroit    \\ 
                               & $\lambda$   & \{0.03,0.05,0.1\} (Detail for each task see Table \ref{tab:hyper_detail})  \\ \hline
\multirow{3}{*}{Architecture}      & Actor                         & input-256-256-action dim           \\
                                   & Critic                        & input-256-256-1                    \\
                                   & Encoder $p_\vartheta(z|s,a)$  & input-750-750-$1.2\times$state dim \\ 
                                   & Encoder $p_\varphi(z|s)$      & input-750-750-$1.2\times$state dim \\
                                   & Decoder $q_\psi(s'|z)$        & input-750-750-state dim     \\ \midrule
\end{tabular}\label{tab:hyper}
\end{table}

\begin{table*}[h]
\centering
\caption{Detailed configurations of $\lambda$ and $\beta$ across different tasks.}
\begin{tabular}{ll|rr|r}
\toprule     
 Task Name           &    & $\lambda$    & $\beta$    & $\lambda_{ENSP}$ \\ \midrule
\multirow{3}{*}{halfcheetah} & medium           & 0.03  & 0.001  &0.05\\
                             & medium-replay    & 0.03  & 0.001  &0.05\\
                             & medium-expert    & 0.05  & 0.001  &0.2\\ \hline
\multirow{3}{*}{hopper}      & medium           & 0.03  & 0.001  &0.1\\
                             & medium-replay    & 0.05  & 0.001  &0.05\\
                             & medium-expert    & 0.05  & 0.001  &0.2\\ \hline
\multirow{3}{*}{walker2d}    & medium           & 0.1   & 0.001  &0.2\\
                             & medium-replay    & 0.1   & 0.001  &0.2\\
                             & medium-expert    & 0.1   & 0.001  &0.2\\ \hline
\multirow{6}{*}{antmaze} & umaze                & 0.1   & 0.01  &0.1\\
                             & umaze-diverse   & 0.05  & 0.01  &0.1\\
                             & medium-play      & 0.05  & 0.01  &0.05\\
                             & medium-diverse   & 0.05  & 0.01  &0.05\\
                             & large-play       & 0.05  & 0.01  &0.05\\
                             & large-diverse    & 0.05  & 0.01  &0.05\\ \hline
pen                          & human            & 0.03  & 0.01  &0.05\\
hammer                       & human            & 0.03  & 0.01  &0.1\\
door                         & human            & 0.03  & 0.01  &0.2\\ \bottomrule
\end{tabular}
\label{tab:hyper_detail}
\end{table*}

\paragraph{Hyperparameter tuning and Learning curves.}
Approaches to offline policy evaluation and selection have garnered significant attention within the research community. As discussed by \citet{tune,spot}, assuming limited access to a real system to test a small number of trained policies constitutes a reasonable setup for practical applications such as robotics. 

In StBR, the coefficient $\lambda$ in Eq.~\eqref{eq:pi_practical} is a crucial hyperparameter that requires tuning. Following prior works~\citep{spot,bear,brac}, we allow limited access to the online environment to tune this parameter across a minimal candidate set of $\{0.03, 0.05, 0.1\}$. To ensure a rigorous evaluation, we tune $\lambda$ using a single random seed. Once the optimal hyperparameter is identified, we train the policy from scratch using five entirely new seeds. The final reported results are based exclusively on these five evaluation seeds, omitting the initial tuning seed. Note that the Adroit Human domain presents a challenging setting characterized by high-dimensional spaces and limited dataset sizes (only $\sim$10,000 transitions). This leads to severe value overestimation, necessitating a stronger degree of conservatism. To facilitate hyperparameter tuning, we specifically employ the following policy loss for this task: $ \mathcal{L}_\pi(\phi) = \mathbb{E}_{s \sim \mathcal{D},a \sim \pi_\phi(\cdot|s)} \left[ -\lambda \frac{Q_\theta(s, a)}{\alpha} + (1-\lambda) D_{\mathrm{KL}}(p_\vartheta(z|s, a) \| p_\varphi(z|s)) \right]. $

Figures \ref{fig:curve1} and \ref{fig:curve2} present the training curves for all tasks in the benchmark experiments. The solid lines depict the mean scores across the five evaluation seeds, while the shaded regions indicate the minimum and maximum values. A natural question that arises is \textit{how latent divergence performs when applied solely to BC.} To address this, these figures also include the training curves for this specific configuration, denoted as StBR+BC.

\begin{figure*}[h]
    \centering
    \includegraphics[width=0.9\linewidth]{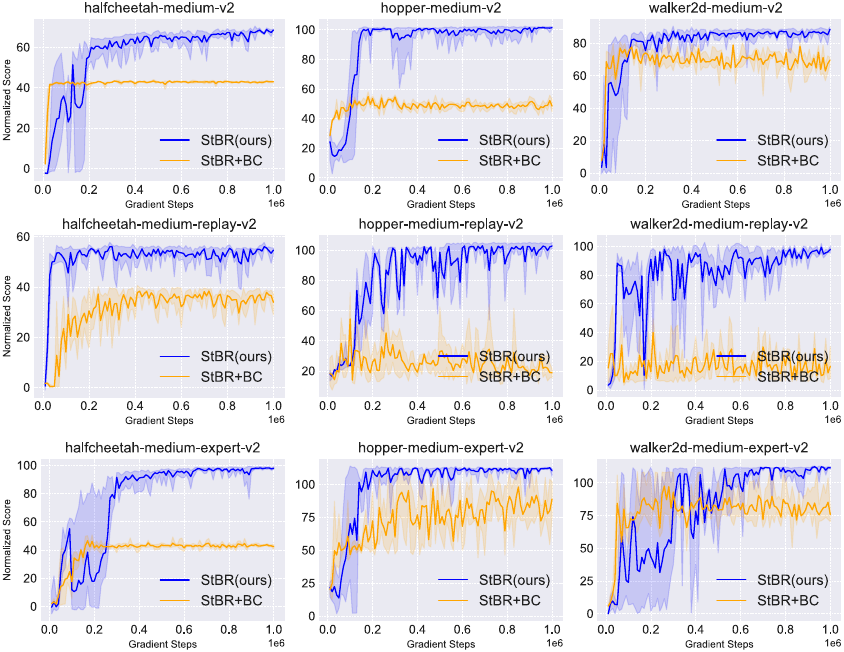}
    \caption{The results of StBR and StBR+BC on the MuJoCo task. The shaded area indicates the minimum and maximum score across 5 seeds, with the solid line depicting their mean.}
    \label{fig:curve1}
\end{figure*}
\begin{figure*}[h]
    \centering
    \includegraphics[width=0.9\linewidth]{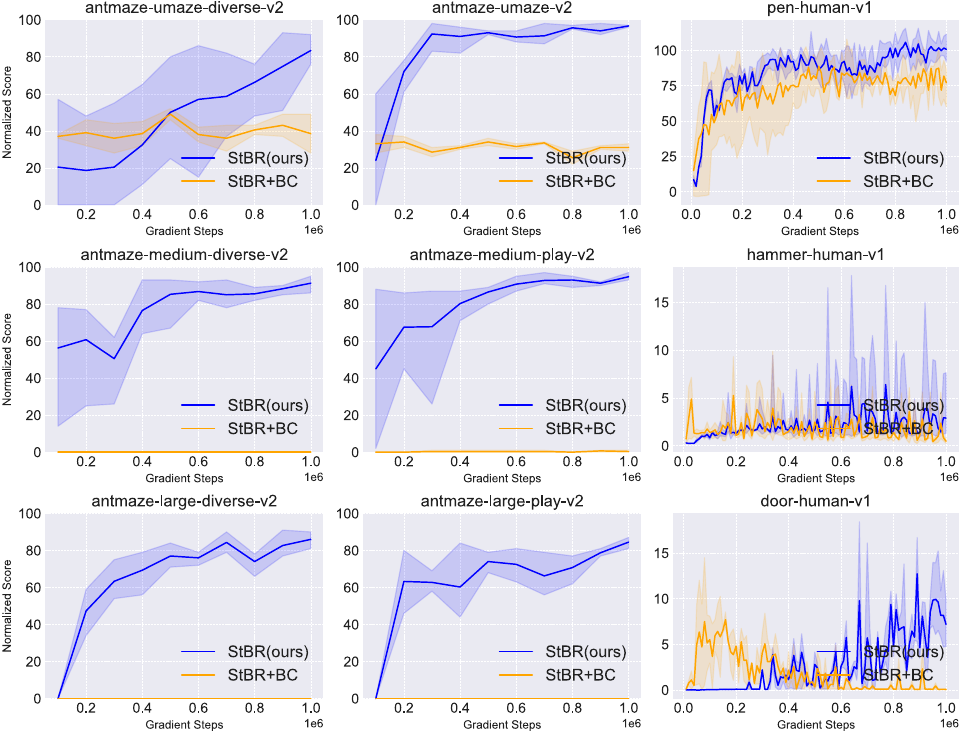}
    \caption{The results of StBR and StBR+BC on the AntMaze and Adroit task. The shaded area indicates the minimum and maximum score across 5 seeds, with the solid line depicting their mean.}
    \label{fig:curve2}
\end{figure*}

\subsection{Theoretical Validation Experiments}\label{sec:app_theoretical}

\paragraph{Design of environment stochasticity experiments} \label{sec:app_theoretical_stochastic}
We evaluate StBR on the \texttt{hopper-medium} task with artificially injected action execution stochasticity.
To simulate transition stochasticity in real-world environments, we introduce action noise during both the offline data collection and online interaction phases as simulated motor errors. Specifically, we add Gaussian noise to the agent's intended action to emulate the control errors inherent in physical actuators $\tilde{a} \sim \mathcal{N}(a, \sigma_{env}^2)$. From the agent's perspective, this discrepancy between the intended action and the executed action leads to unpredictable next states, thereby naturally transforming the deterministic Hopper simulator into a stochastic environment.

\paragraph{Visualization of OOD action separation.} 
Furthermore, we evaluate whether StBR effectively shapes the divergence landscape to prevent the agent from taking OOD actions. Figure~\ref{fig:visual_ood} tracks the latent divergence throughout the training process for both OOD actions ($a \sim \mathrm{Unif}(\mathcal{A})$) and learned policy's actions ($a \sim \pi_\phi$) in \texttt{hopper-medium} and \texttt{pen-human} tasks. As illustrated in Figure~\ref{fig:visual_ood}, the latent divergence for OOD actions remains substantially larger than that of the policy-constrained actions. This confirms that StBR successfully penalizes minor action deviations that could lead to drastically different, unknown transitions, thereby robustly keeping the policy within the offline dataset.

\begin{figure*}[h]
    \centering
    \includegraphics[width=0.75\linewidth]{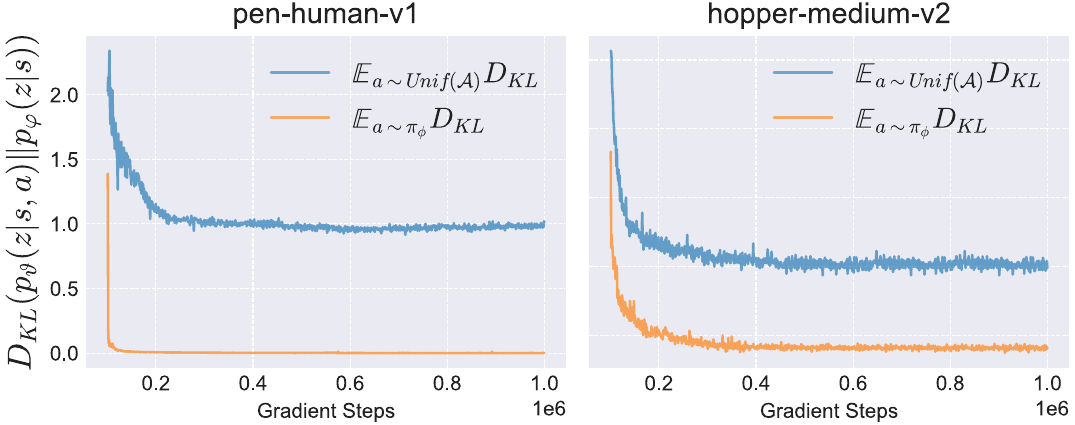}
    \caption{The figure tracks the latent divergence across gradient steps for OOD actions and learned policy's actions. Notably, the latent divergence of the constrained policy is significantly lower than that of the OOD actions, demonstrating that our latent divergence effectively separates in-distribution (ID) from OOD actions.}
    \label{fig:visual_ood}
\end{figure*}

\subsection{Additional Experimental Results}\label{sec:app_additional}
To comprehensively validate the capabilities and practical advantages of StBR beyond standard benchmark evaluations, we design three targeted sets of supplementary experiments. Specifically, we aim to answer the following questions: (1) How does StBR perform when high-quality data are scarce? (2) Does the consequence-aware nature of StBR practically enable out-of-sample (OOS) state recovery ability? and (3) How does StBR differ from model-based approaches when evaluating its learned decoder?

\paragraph{Robustness to limited valuable data.}
To investigate our StBR's performance on suboptimal datasets, we design an experiment to quantitatively evaluate scenarios with limited valuable data.
Specifically, we construct mixed datasets using \texttt{walker2d-random} and \texttt{walker2d-expert}. By progressively reducing the proportion of expert data from 50\% down to 10\%, we simulate scenarios where high-quality transitions are extremely scarce.

As shown in Table~\ref{tab:limited}, severely limiting the expert data (e.g., down to 10\%) causes traditional methods (CQL, SDC) to experience a performance collapse, largely due to over-conservatism affected by suboptimal transitions. While the advanced FQL achieves stronger robustness through more expressive modeling of the behavior policy, StBR consistently outperforms all these methods across all levels. This superior robustness is directly enabled by our consequence-aware support constraint.

\begin{table}[h]
    \centering
    \caption{Normalized scores on mixed walker2d datasets with varying proportions of random (r) and expert (e) transitions. Results are averaged over four seeds. StBR demonstrates significant robustness against limited valuable data  dilution.}
    \label{tab:limited}
    \begin{tabular}{@{} l|c c c |c@{}}\toprule
Limited Valuable Dataset   & CQL      & SDC      & FQL       & StBR       \\\midrule
Walker2d-r0.5-e0.5      & 88.3±4.9 & 88.8±3.9 & 108.9±4.5 & \textbf{113.9}±3.1 \\
Walker2d-r0.6-e0.4      & 81.5±6.3 & 65.9±5.0 & 91.7±6.6  & \textbf{106.1}±2.4 \\
Walker2d-r0.7-e0.3      & 48.3±5.8 & 85.5±4.7 & 88.7±9.1  & \textbf{99.9}±8.1  \\
Walker2d-r0.8-e0.2      & 50.5±7.2 & 49.8±5.6 & 70.4±11.5 & \textbf{87.7}±5.5  \\
Walker2d-r0.9-e0.1      & 17.6±3.0 & 26.8±2.1 & 66.7±23.6 & \textbf{73.4}±22.9 \\ \bottomrule
    \end{tabular}
\end{table}

\paragraph{Dataset-size ablation.}
We evaluated StBR on the hopper-medium dataset under varying data sizes, with the mean and standard deviation reported across 5 seeds. We randomly sampled trajectory subsets ranging from 1\% to 100\% of the offline dataset for standard offline RL training. As motivated by our theoretical analysis, a larger $\beta$ is beneficial under data scarcity; thus, we used $\beta = 0.003$ for dataset-size below 10\% and $\beta = 0.001$ for the rest. We also compared StBR against a VAE-based regularization baseline (SPOT). 

The results in Table~\ref{tab:size_vary} demonstrate that StBR is robust to dataset size, maintaining strong performance and stability when at least 10\% of the data is available. When the dataset is reduced to 5\% and 1\%, where the primary bottleneck becomes the quality of the randomly sampled trajectories, StBR still achieves stable learning. 

\begin{table}[h]
    \centering
    \caption{Normalized scores with varying dataset-size. Results are averaged over 5 seeds.}
    \label{tab:size_vary}
    \begin{tabular}{@{} l|c c c c c c@{}}\toprule
Dataset-size(\%)   & 100&50&25&10&5&1    \\\midrule
StBR(Ours)      & 100.9±1.0&100.4±1.2&99.5±1.5&97.2±4.1&78.2±2.7&65.6±14.5 \\
SPOT      & 86.0±8.7&83.2±14.0&83.6±6.2&72.9±4.5&44.1±15.0&38.8±19.0 \\ 
\bottomrule
    \end{tabular}
\end{table}

\paragraph{Out-of-sample (OOS) state recovery.}
We investigate whether StBR equips the agent with the ability to recover from OOS states back to the data manifold. Following the perturbation protocols established in prior OOS state recovery studies~\citep{sdc,osr}, we conduct an OOS state recovery experiment on the AntMaze tasks. Specifically, we apply varying scales of Gaussian noise $s_{oos}\sim \mathcal{N}(s, \sigma_{oos}^2)$ to the in-distribution states $s$ sampled from the offline dataset, yielding perturbed OOS states $s_{oos}$. We then train the policy to pull the agent back to the behavior support from these OOS states by the following loss:
\begin{align}
\mathcal{L}_\pi(\phi) = \mathbb{E}_{s \sim \mathcal{D}}\mathbb{E}_{s_{oos} \sim \mathcal{N}(s, \sigma_{oos}^2)} \big[ &-\mathbb{E}_{a \sim \pi_\phi(\cdot|s)} Q_\theta(s, a) \nonumber \\
&+ \lambda \mathbb{E}_{a' \sim \pi_\phi(\cdot|s_{oos})}D_{\mathrm{KL}}(p_\vartheta(z|s_{oos}, a') \| p_\varphi(z|s_{oos})) \big].
\end{align}

As illustrated in Table~\ref{tab:osr}, introducing an appropriate level of OOS state perturbation (e.g., $\sigma_{oos}=0.001$) effectively enhances the agent's generalization performance. However, when the perturbation scale is excessively large (e.g., $\sigma_{oos}=0.1$), it completely distorts the underlying semantic structure of the states, making the OOS states unrecoverable and resulting in severe performance degradation. Overall, these results verify that StBR successfully endows the agent with remarkable state recovery capabilities, robustly navigating it back to safe regions when facing moderate distributional shifts.



\begin{table}[h]
\centering
\caption{The mean and standard deviation of averaged normalized scores after 1M steps of OOS state recovery training, averaged over 5 seeds and the final 4 evaluations. The highest scores in each task are bolded.}
\label{tab:osr}
\begin{tabular}{lllll} \toprule
Task Name           & $\sigma_{oos}=0$        & $\sigma_{oos}=0.001$    & $\sigma_{oos}=0.01$     & $\sigma_{oos}=0.1$      \\\midrule
AntMaze-medium-play      & 94.2±1.9 & \textbf{97.1}±2.4 & 92.6±3.1 & 44.4±7.4 \\
AntMaze-medium-diverse   & 90.6±3.6 & \textbf{95.1}±1.8 & 91.0±4.5 & 47.4±4.8 \\
AntMaze-large-play       & 84.2±2.7 & \textbf{88.5}±2.7 & 84.5±2.2 & 19.6±6.2 \\
AntMaze-large-diverse    & 85.8±5.1 & \textbf{86.2}±3.3 & 83.1±6.1 & 12.8±7.4 \\\midrule
Total                    & 354.8 &	354.8$\rightarrow$\textbf{366.9} &	354.8$\rightarrow$351.2 &	354.8$\rightarrow$124.2 \\ \bottomrule
\end{tabular}
\end{table}

\paragraph{Distinction from model-based methods.}
Since SBGF learns a decoder that can essentially be viewed as a complete dynamics model, it is natural to ask whether our approach can be explicitly deployed as a model-based offline RL method. To investigate this, we implement an StBR+MB algorithm built upon the MOPO~\citep{mopo} framework. Specifically, we utilize the trained decoder as a forward model to generate rollout data and employ our latent divergence as the measure for model uncertainty with the same hyperparameter tuning protocol as model-free StBR. 

The empirical results demonstrate that StBR+MB can indeed serve as an effective model-based variant, achieving performance slightly superior to MOPO. However, its overall efficacy falls short of more advanced model-based methods like RAMBO, primarily due to the inherent parameter sensitivity and the difficulty of acquiring perfectly reliable model uncertainty estimates. 
More importantly, our standard model-free StBR significantly outperforms all of these model-based baselines. This highlights a critical advantage: even if the learned forward model is imperfect for data generation, the latent space successfully acquires consequence-aware representations. This allows the latent divergence to provide a highly effective, robust constraint on the policy without relying on explicit rollouts, fundamentally distinguishing the mechanism of StBR from traditional model-based paradigms.



\begin{table}[h]
\centering
\caption{The mean and standard deviation of averaged normalized scores after 3M steps of model-based offline RL training, averaged over 5 seeds and the final 10 evaluations. The top-2 highest scores in each task are bolded.}
\label{tab:mb_results}
\begin{tabular}{ll|ll|ll}
\toprule
                             & Task Name      & MOPO      & RAMBO     & StBR+MB    & StBR(ours)       \\\midrule
\multirow{3}{*}{halfcheetah} & medium        & 42.3±1.6  & \textbf{71.0}±2.9  & 41.4±2.8  & \textbf{60.7}±0.7  \\
                             & medium-replay & 53.1±2.0  & \textbf{67.0}±1.5  & \textbf{55.8}±3.1  & 54.4±1.0  \\
                             & medium-expert & 63.3±38.0 & 79.3±2.9  & \textbf{90.6}±5.9  & \textbf{97.8}±0.2  \\\hline
\multirow{3}{*}{hopper}      & medium        & 28.0±12.4 & \textbf{91.2}±16.3 & 80.3±7.3  & \textbf{100.9}±1.0 \\
                             & medium-replay & 67.5±24.7 & \textbf{97.6}±3.4  & 90.6±2.6  & \textbf{100.8}±1.9 \\
                             & medium-expert & 23.7±6.0  & \textbf{89.5}±11.1 & 77.8±6.5  & \textbf{111.6}±0.9 \\\hline
\multirow{3}{*}{walker2d}    & medium        & 17.8±19.3 & \textbf{89.1}±2.7  & 7.8±6.0   & \textbf{87.2}±0.9  \\
                             & medium-replay & 39.0±9.6  & \textbf{88.5}±4.0  & 41.5±8.4  & \textbf{86.3}±1.5  \\
                             & medium-expert & 44.6±12.9 & \textbf{63.1}±31.3 & 46.7±23.7 & \textbf{111.3}±1.4 \\\midrule
Total                        &               & 379.3     & \textbf{747.5}     & 532.5     & \textbf{811.0}     \\
\bottomrule
\end{tabular}
\end{table}

\begin{table}[h]
\centering
\caption{The runtime of StBR on \texttt{hopper-medium-v2} task using a GeForce RTX 2080Ti GPU and Intel Xeon Gold 5218 CPU.}
\resizebox{\linewidth}{!}{
\begin{tabular}{l|rrrrrrrrr|r}
\toprule
                   & BC   & CQL  & TD3BC  & IQL  & PBRL  & POR  & SDC  & SPOT  & FQL  & StBR(ours) \\ \midrule
Runtime (ms/step)  & 2.0  & 24.2 & 4.7    & 6.5  &  26.7 & 8.8  & 25.0 & 6.7   & 6.6  & 6.8   \\
Wall time (minutes)& 33   & 403  & 78     & 108  &  444  & 147  & 417  & 111   & 110  & 114   \\
Inference time (ms)& -    & -    & 476    & -    &  -    & 909  & -    & -     & 492  & 476   \\
 \bottomrule
\end{tabular}
}
\label{tab:computation_results}
\end{table}

\subsection{Computation Cost}
Benefiting from its pluggable design, StBR operates as a highly efficient regularization method. During inference, it requires only a single forward pass of the policy network. During the training phase, StBR introduces minimal computational overhead, necessitating only the additional forward passes of two encoders. This provides a distinct computational advantage over ensemble-based methods, which require maintaining and querying multiple models, as well as flow-matching or diffusion models, which rely on numerous iterative forward passes.

As demonstrated in Table~\ref{tab:computation_results} on the \texttt{hopper-medium-v2} task, StBR maintains a wall-clock training time that is highly competitive with most standard regularization methods. Furthermore, because StBR trains a single-step policy and integrates seamlessly as a pluggable module, it achieves exceptionally fast inference speeds, making it highly practical for real-time deployment.


\begin{table}[h]
    \centering
    \caption{Online fine-tuning results on AntMaze tasks, showing initial performance after offline RL and performance after 1M steps of online RL. All results are averaged over 5 seeds.}
    \label{tab:off2on}
    \begin{tabular}{@{} l|c c |c@{}}
        \toprule
        AntMaze          & IQL & Cal-QL & StBR (Ours) \\ \midrule
        umaze            & 89.6 $\rightarrow$ 96.2    & 76.8 $\rightarrow$ \textbf{99.8}  & 99.2 $\rightarrow 99.6_{\text{±}0.4}$ \\
        umaze-diverse    & 65.6 $\rightarrow$ 87.2    & 32.0 $\rightarrow$ 98.5  & 95.6 $\rightarrow \textbf{99.4}_{\text{±}0.7}$ \\
        medium-play      & 76.4 $\rightarrow$ 94.0    & 71.8 $\rightarrow$ 98.8  & 93.6 $\rightarrow \textbf{99.0}_{\text{±}0.8}$ \\
        medium-diverse   & 72.8 $\rightarrow$ 92.0    & 62.0 $\rightarrow$ 98.3  & 89.8 $\rightarrow \textbf{99.1}_{\text{±}0.7}$ \\
        large-play       & 42.0 $\rightarrow$ 51.0    & 31.8 $\rightarrow$ \textbf{97.3}  & 78.0 $\rightarrow 96.9_{\text{±}1.1}$ \\
        large-diverse    & 46.0 $\rightarrow$ 58.4    & 44.0 $\rightarrow$ 91.5  & 77.8 $\rightarrow \textbf{96.5}_{\text{±}1.2}$ \\ \midrule
        \textbf{Average} & 65.4 $\rightarrow$ 79.8    & 53.0 $\rightarrow$ 97.3  & 89.0 $\rightarrow \textbf{98.4} $    \\ \bottomrule
    \end{tabular}
\end{table}

\begin{figure}[h]
    \centering
    \includegraphics[width=0.8\linewidth]{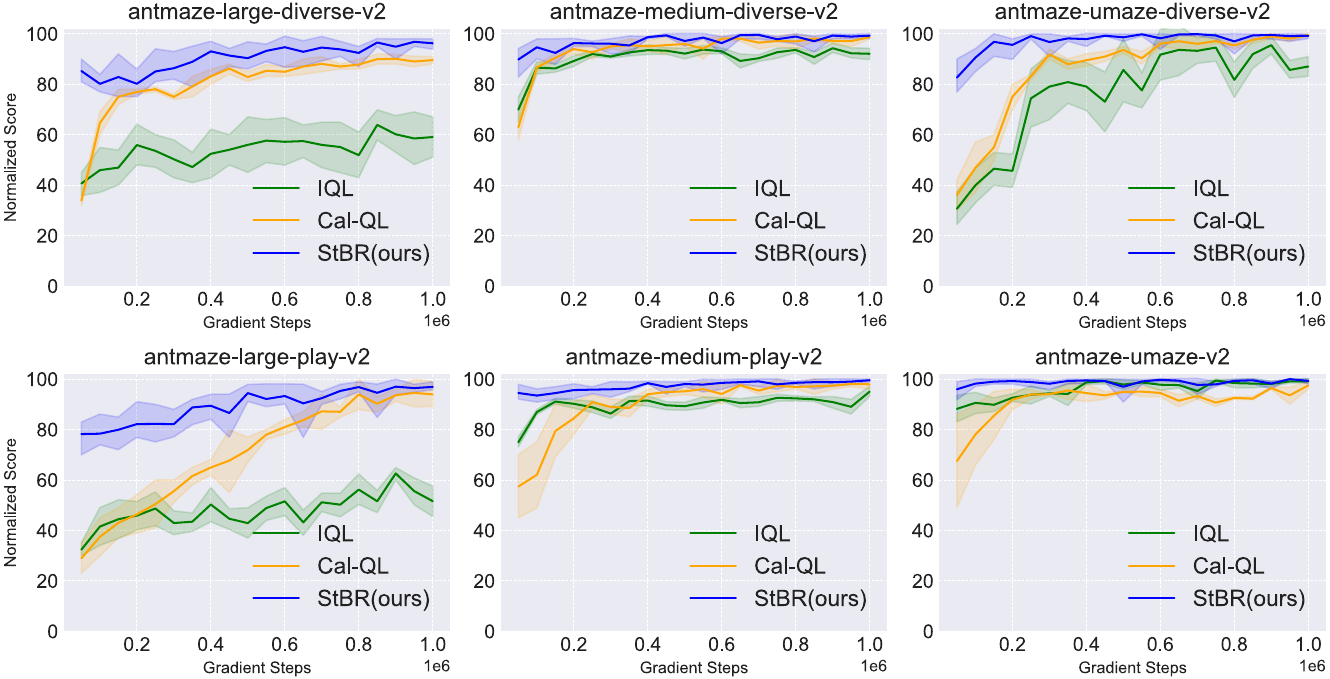}
    \caption{Online fine-tuning training curves on AntMaze tasks, showing initial performance after offline RL and performance after 1M steps of online RL. (5 seeds)}
    \label{fig:off2on}
\end{figure}

\subsection{Can StBR be fine-tuned with online rollouts?}\label{sec:app_off2on}
StBR naturally extends to the online fine-tuning setting. This is achieved by continuously appending new online transitions, $\mathcal{D}_{\text{online}}$, to the replay buffer while training the policy and value networks with the original offline objective.

\paragraph{Implementation details.} 
We conduct our online fine-tuning experiments on the complex AntMaze domains. Due to their high-dimensional state-action spaces and sparse rewards, bootstrapping errors can easily accumulate during online exploration. To mitigate this, we retain the complete offline dataset in the replay buffer and augment it with new online data, which is actively collected using an exploration noise of $0.1$. Notably, without altering the algorithmic framework, we continuously update the SBGF throughout the $1$ million gradient steps of the online phase. All other hyperparameters remain identical to the baseline methods.

\paragraph{Comparison results.} 
We evaluate StBR against Cal-QL~\citep{calql}, a performant method specifically designed for offline-to-online RL and IQL~\citep{iql}, a standard offline RL baseline. Table~\ref{tab:off2on} summarizes the final performance, and Figure~\ref{fig:off2on} plots the corresponding learning curves. The results demonstrate that the superior offline policy provided by StBR serves as a robust initialization, allowing the agent to achieve faster convergence and higher asymptotic performance compared to existing approaches.

\newpage

\section*{NeurIPS Paper Checklist}

\begin{enumerate}

\item {\bf Claims}
    \item[] Question: Do the main claims made in the abstract and introduction accurately reflect the paper's contributions and scope?
    \item[] Answer: \answerYes{} 
    \item[] Justification: The "abstract", "introduction" and "preliminaries and motivation" sections clearly state the problem of consequence-shift vulnerability, and the methodologies to solve this effectively. Theoretical and experimental evidences are provided to support our claims.

\item {\bf Limitations}
    \item[] Question: Does the paper discuss the limitations of the work performed by the authors?
    \item[] Answer: \answerYes{}  
    \item[] Justification: The paper includes a dedicated "Limitations" paragraph in the Conclusion section.

\item {\bf Theory assumptions and proofs}
    \item[] Question: For each theoretical result, does the paper provide the full set of assumptions and a complete (and correct) proof?
    \item[] Answer: \answerYes{} 
    \item[] Justification: The main paper and Appendix explicitly state the required assumptions.

    \item {\bf Experimental result reproducibility}
    \item[] Question: Does the paper fully disclose all the information needed to reproduce the main experimental results of the paper to the extent that it affects the main claims and/or conclusions of the paper (regardless of whether the code and data are provided or not)?
    \item[] Answer: \answerYes{}  
    \item[] Justification: Detailed experimental settings are referenced in the main text. These are extensively expanded upon in Appendix C, providing the necessary architectural configurations, hyperparameters, optimizers, and extensive supplementary experiments.

\item {\bf Open access to data and code}
    \item[] Question: Does the paper provide open access to the data and code, with sufficient instructions to faithfully reproduce the main experimental results, as described in supplemental material?
    \item[] Answer:  \answerYes{}
    \item[] Justification: The code were supplied to reviewers as an anonymous zip archive, and full
public links will be added in the camera-ready version.

\item {\bf Experimental setting/details}
    \item[] Question: Does the paper specify all the training and test details (e.g., data splits, hyperparameters, how they were chosen, type of optimizer) necessary to understand the results?
    \item[] Answer: \answerYes{} 
    \item[] Justification: The paper specifies baseline architectures, optimization protocols, hardware details, and hyperparameters comprehensively in Experiment section and Appendix C.

\item {\bf Experiment statistical significance}
    \item[] Question: Does the paper report error bars suitably and correctly defined or other appropriate information about the statistical significance of the experiments?
    \item[] Answer: \answerYes{} 
    \item[] Justification:  Our paper reports appropriate information about statistical significance of the experiments.

\item {\bf Experiments compute resources}
    \item[] Question: For each experiment, does the paper provide sufficient information on the computer resources (type of compute workers, memory, time of execution) needed to reproduce the experiments?
    \item[] Answer: \answerYes{} 
    \item[] Justification: Appendix C details the specific compute hardware used (GeForce RTX 2080Ti GPU and Intel Xeon Gold 5218 CPU).
    
\item {\bf Code of ethics}
    \item[] Question: Does the research conducted in the paper conform, in every respect, with the NeurIPS Code of Ethics \url{https://neurips.cc/public/EthicsGuidelines}?
    \item[] Answer: \answerYes{}  
    \item[] Justification: The research involves fundamental algorithmic developments evaluated on standard, publicly available datasets (e.g., D4RL) and does not violate any ethical guidelines.

\item {\bf Broader impacts}
    \item[] Question: Does the paper discuss both potential positive societal impacts and negative societal impacts of the work performed?
    \item[] Answer: \answerNA{} 
    \item[] Justification: This paper focuses on foundational theoretical improvements to behavior regularization offline RL algorithms. It does not present direct, foreseeable negative societal impacts or obvious ties to malicious applications.
    
\item {\bf Safeguards}
    \item[] Question: Does the paper describe safeguards that have been put in place for responsible release of data or models that have a high risk for misuse (e.g., pre-trained language models, image generators, or scraped datasets)?
    \item[] Answer: \answerNA{} 
    \item[] Justification: The models and evaluation benchmarks used do not fall under the category of high-risk for dual-use or misuse. No scraped or sensitive data is released.

\item {\bf Licenses for existing assets}
    \item[] Question: Are the creators or original owners of assets (e.g., code, data, models), used in the paper, properly credited and are the license and terms of use explicitly mentioned and properly respected?
    \item[] Answer: \answerYes{} 
    \item[] Justification: Existing datasets, platforms, and baseline models are properly credited with citations in the text and the references section.

\item {\bf New assets}
    \item[] Question: Are new assets introduced in the paper well documented and is the documentation provided alongside the assets?
    \item[] Answer: \answerNA{} 
    \item[] Justification: The paper does not introduce or release new datasets or proprietary assets.

\item {\bf Crowdsourcing and research with human subjects}
    \item[] Question: For crowdsourcing experiments and research with human subjects, does the paper include the full text of instructions given to participants and screenshots, if applicable, as well as details about compensation (if any)? 
    \item[] Answer:  \answerNA{} 
    \item[] Justification: This research does not involve any human subjects or crowdsourcing.

\item {\bf Institutional review board (IRB) approvals or equivalent for research with human subjects}
    \item[] Question: Does the paper describe potential risks incurred by study participants, whether such risks were disclosed to the subjects, and whether Institutional Review Board (IRB) approvals (or an equivalent approval/review based on the requirements of your country or institution) were obtained?
    \item[] Answer: \answerNA{} 
    \item[] Justification: There is no research involving human subjects, making IRB approval unnecessary.

\item {\bf Declaration of LLM usage}
    \item[] Question: Does the paper describe the usage of LLMs if it is an important, original, or non-standard component of the core methods in this research? Note that if the LLM is used only for writing, editing, or formatting purposes and does \emph{not} impact the core methodology, scientific rigor, or originality of the research, declaration is not required.
    \item[] Answer: \answerNA{} 
    \item[] Justification: Large language models are not part of the paper’s core methods.

\end{enumerate}

\end{document}